\documentclass{article}
\usepackage[paper=letterpaper,centering,margin=1.0in]{geometry}

\usepackage{authblk}

\usepackage{bm}
\usepackage{xcolor}
\usepackage[ruled,vlined]{algorithm2e}

\usepackage{dsfont}
\usepackage{bbm}
\usepackage{bm}
\usepackage[mathscr]{eucal}
\usepackage{mathtools}
\usepackage[]{color}
\usepackage{tabularx}
\newcolumntype{Y}{>{\centering\arraybackslash}X}

\usepackage{caption}
\usepackage{subcaption}
\usepackage{mwe}
\usepackage{tikz}
\usetikzlibrary{positioning}
\usepackage[ruled]{algorithm2e}
\usepackage{natbib}
 \bibpunct[, ]{(}{)}{,}{a}{}{,}%

\usepackage{enumerate}
\def\={\!=\!}

\usepackage[normalem]{ulem}

\newcommand{\Beginproof}{\begin{proof}}
\newcommand{\Endproof}{\end{proof}}

\usepackage{hyperref}
\hypersetup{colorlinks,citecolor=blue,urlcolor=blue,linkcolor=blue}
\usepackage{kx}

\usepackage{amssymb,amsmath,amsthm}
\usepackage{graphicx}
\newtheorem{theorem}{Theorem}
\newtheorem{corollary}{Corollary}
\newtheorem{lemma}{Lemma}
\newtheorem{proposition}{Proposition}
\DeclareMathOperator*{\argmax}{argmax}

\begin{document}

\title{Learning and Information in Stochastic Networks and Queues}

\author{Neil Walton}
\affil{University of Manchester\\
\small{\texttt{neil.walton@manchester.ac.uk}}}

	\author{Kuang Xu}
\affil{Stanford University\\
\small{\texttt{kuangxu@stanford.edu}}}

\maketitle

\abstract{
We review the role of information and learning in the stability and optimization of queueing systems. 
In recent years, techniques from supervised learning, bandit learning and reinforcement learning have been applied to queueing systems supported by increasing role of information in decision making. 
We  present observations and new results that help rationalize the application of these areas to queueing systems. 

We prove that the MaxWeight and BackPressure policies are an application of Blackwell's Approachability Theorem. This connects queueing theoretic results with adversarial learning. We then discuss the requirements of statistical learning for service parameter estimation. As an example, we show how queue size regret can be bounded when applying a perceptron algorithm to classify service. Next, we discuss the role of state information in improved decision making. Here we contrast the roles of epistemic information (information on uncertain parameters) and aleatoric information (information on an uncertain state). Finally we review recent advances in the theory of reinforcement learning and queueing, as well as, provide discussion on current research challenges. 
}

\maketitle

\section{Introduction}

This tutorial aims to summarize the role of learning and information in queueing systems. 
Over the past decade, the theory of machine learning and data science has grown rapidly. 
As a consequence, techniques from these areas have infused themselves within a variety of operations research disciplines.
The scope for applying these techniques to queueing networks is vast and important given the role of queueing networks in areas such communications, manufacturing, healthcare, supply chain and transportation.
However, it is reasonable to say that the combination of learning and queueing theory is still in its infancy, particularly when compared with other areas of operations research, such as  revenue management, where there are well-established connections between parameter estimation and decision making. %Often it is possible to apply algorithms 

When we control a queueing systems, stability is usually of first order concern: can we provision service in order to cope with demand. 
Performance is then the next concern. It might seem that, to control a queueing system well, we need to know the arrival rate and service rate parameters, because we need to be able to solve for an optimal policy as a mapping from the state to the control action. In many practical applications, the state may not be observable, or may be costly to observe. Further errors in estimating arrival and service parameters may also incur costs. In Section \ref{BAMW}, we look at the question of controlling a queueing system with unknown arrival rates. A classical result shows that a simple policy, called the MaxWeight policy, stabilizes the system and is agnostic to the arrival rates. We show that this policy is intimately tied to the classical result known as Blackwell Approachability, which in turn has an intimate connection with the theory of Online Convex Optimization. In Section \ref{RegretQ}, we further develop the use of online learning algorithms in queueing systems. The quality of an online learn algorithm is typically judged in-terms of its regret. We discuss how bounds on regret translate into bounds on the performance of a queueing model. This assesses the impact of estimating service on a queues performance. In Section \ref{information}, we assess the impact of state information on queueing network control. We discuss how, in two prior works \citep{spencer2014queueing,gamarnik2018delay}, increased state information can dramatically improve performance. 

Finally in Section \ref{sec:RL}, we review recent progress applying reinforcement learning to queueing systems. 
From a practical perspective, it is often possible to apply existing techniques, from reinforcement learning, to a queueing system. Indeed, as we will review, some of the earliest examples applying Q-learning with neural networks are for queueing systems. 
However from a theoretical perspective, the assumptions of these learning frameworks often do not naturally match with the assumptions of a queueing theoretic model. 
For instance, in reinforcement learning theory, it is common to assume a finite state-space and bounded rewards which are assumptions that often do not hold for a queueing systems (see Section \ref{RLChall} for further discussion). 
Moreover, in instances where assumptions do hold, estimation, modeling and optimization tasks of can often be treated separately. For instance, parameters of a queueing network, such as arrival rates and routing probabilities, can be estimated statistically. These can be applied to a queueing formula to estimate stationary queue lengths for a fixed capacity allocation. Given this formula service capacity can then be allocated to optimize to queue lengths. 
Nonetheless, there is a great deal of potential for reinforcement learning techniques to be applied to a queueing systems. 
Moreover, the analysis of queueing systems could greatly benefit from the more exploratory research approach taken in reinforcement learning (see \cite{dietterich1990editorial}). From this there are
a growing number of practical and important theoretical works emerging in this area which we highlight. 
%models do not naturally match the assumptions  or if they do then the two problems  .
%
%

With this noted, the aim of this article is to introduce the reader to a few techniques and insights which we believe are important for understanding the role of learning and information in queueing systems, as well as potential challenges. 
%

%We do this first by providing two examples of theoretical results applied to a queueing systems (Section \ref{BAMW} and Section \ref{RegretQ}). These initial examples provides motivation from which we can then survey recent work incorporating information within queueing network analysis (Section \ref{information}). Toward the end of this tutorial, we discuss the challenges in developing techniques from reinforcement learning for queueing systems (Section \ref{sec:RL}). %Here we find that it is often possible to find practical schemes through implementing these algorithms; however, the current theory of reinforcment learning which is 

\subsection{Overview.}

We provide a more detailed overview of the results, discussions and approaches taken in each section of this article.

\subsubsection*{Blackwell Approachability and MaxWeight.}
In Section \ref{BAMW}, we show that a well-known stochastic network scheduling algorithms is in fact a learning algorithm. Specifically, we analyze the MaxWeight scheduling algorithm in switched queueing networks. Switched queueing networks are a discrete-time queueing network model where there are constraints on which queues can be served simultaneously (Section \ref{subsec:switch}). Switched queueing networks are, for instance, used to model internet routers, wireless communication systems, data centers, and manufacturing systems. In a switch queueing network, a scheduling algorithm decides which queues are serve at each time step. In this setting, the MaxWeight algorithm is a celebrated scheduling algorithm. MaxWeight has the desirable property that it has a maximal stability region (Theorem \ref{MW:thrm}). This means that, regardless of the arrival rate of jobs, if there is a way to stabilize the network then MaxWeight will stabilize the network. Importantly MaxWeight achieves this with knowledge of the underlying arrival rates. 

We establish the connection between MaxWeight and Blackwell's Approachability Theorem. 
 Blackwell's Approachability Theorem is a classical result in game theory, learning and control. 
Blackwell's result provides the basis for a broad class of online learning algorithms (see \cite{abernethy2011blackwell}). 
\cite{blackwell1954controlled} considers a game between a player and an adversary. The player receives rewards as a vector which depend on its adversaries actions. The task of the player is to make a sequence of decisions so that their average payoff converges towards a convex set, so-called Approachability. 
Blackwell provides a necessary and sufficient condition for approachability as well as an algorithm for the player to approach any convex set. 
In Proposition \ref{MW:Prop}, we show that the maximal stability of MaxWeight can be seen as an instance of approachability in a game between a player who provides service and an adversary who selects arrivals.
This established a direct connection between two classical results in learning and queueing network control.
We note, however, that adversarial service cannot be easily be represented within this framework. This suggests that separate mechanisms are required to learn service parameters. We discuss this in Section \ref{RegretQ}. 

%
%It provides a necessary and sufficient condition for the average vector of rewards in a two-player game to approach a convex set. 
%

\subsubsection*{Online Learning and Queues.} In Section \ref{RegretQ}, we provide a simple example that assesses the impact of parameter learning on the performance a queue.  A number of papers over recent years have investigated the role of learning and regret in queueing. 
Here we consider a model of a single server queue who must sequentially make service decisions which effect the rate of service. 
These service rates are unknown and so
the aggregate performance of different service decisions must be predicted and optimized by an online algorithm. 
The performance of the online learning algorithm is typically measured in terms of its \emph{regret}. The regret is the difference between the aggregate performance of the algorithm compared to aggregate performance of the best decision made in hindsight. An algorithm is deemed to have a good performance if its regret approaches zero at a fast rate.
The existence regret minimizing algorithms is a consequence of Blackwell's Approachability Theorem (Theorem \ref{Hannan}) and for this reason the theorem is considered a foundational result in online learning. 

In Section \ref{RegretQ}, we provide a simple result that applies Lindley's recursion for the G/G/1 queue in order to express queueing size in-terms of the regret of the algorithm used to select service.
In this way we can directly apply the regret analysis of a service rule in order to understand the impact of the algorithm on queue sizes. 
We apply this result for a service rule using the Perceptron Algorithm. The Perceptron Algorithm is a classical algorithm in supervised learning with allows for a particularly tractable analysis in our setting. We prove the Perceptron Mistake Bound (Theorem \ref{PMB}) and apply it to prove an optimality result learning a queue's service (Theorem \ref{QOpt}).

\subsubsection*{Information and Queues.} 
While Section \ref{RegretQ} demonstrates how improved parameter estimation can improve performance, Section \ref{information} discusses how increasing state information can often be more effective. 
For instance, an approximate estimate of the arrival times of customers at a queue over the next 2 minutes may provide more relevant information than a highly accurate estimate of the mean arrival rate over a day. 
Following \cite{lu2021reinforcement}, we categorize information as epistemic information and
 aleatoric information. Epistemic being information on underlying parameters and aleatoric information being information on the underlying system state. 
We review two sets of queueing theory results where the benefits of increased (aleatoric) state information can quantifiably improve performance. 
First we consider future information can assist in reducing delay when making admission decisions at a queue \citep{stidham1985optimal,xu2015necessity} . 
Here we see that there is a sharp cut off between an information-rich regime, which leads to marked queue size reduction, and an information-scarce regime, where performance is comparable with the best epistemic policy, i.e. the best policy that only knows parameters. 
Next we consider a parallel server model from \cite{gamarnik1998stability}. 
Here we review the impacts of limited memory and limited communication. It is found that there are regimes which depend on the availability of memory and rate of messaging which, when sufficiently large, result in zero delay to arriving jobs.  
Again we see that ensuring sufficient state information can have a pronounced effect on performance. 

There is often a natural Markovian description of a queueing system for applying performance analysis.
However, these results emphasize that the state of a queueing system should not be taken as given but instead is an important consideration when designing optimization and learning techniques. 

\subsubsection*{Reinforcement learning and Queues.} 
Finally, in Section \ref{sec:RL}, we review recent progress on the application of reinforcement learning theory to queueing systems. 
We note that many early examples of reinforcement learning (including early deep reinforcement learning) have been applied to scheduling problems in queueing systems. However, while it is possible to apply reinforcement learning algorithms to queueing systems in a relatively straightforward manner, we note that many of the theoretical assumptions that guarantee convergence and correctness in reinforcement learning tend not to hold in the case of a queueing system. 
We itemize and discuss these issues. 
We identify key areas where more theoretical development is required. We then provide a review of recent literature that strives to address this challenge. This includes a number of recent work at the Reinforcement Learning in Networks and Queues workshop at the 2021 ACM Sigmetrics conference. 

While applications of reinforcement learning to queueing systems are still very nascent, our aim is that this tutorial provides an introduction that is both instructive but also acts as a guide on recent literature and potential themes of future research.

%
%Here we consider the task of how to best generate information when we have finite resources. What the is the utility of information to resource allocation. \Neil{Kuang can fill this in :-). Good to get some discussion on information workload etc...}
%
%\Neil{Reinforcment learning and queueing discussion}

\subsection{Organization.}
The remainder of the paper is organized as follows. 
In Section \ref{BAMW} we prove the relationship between the MaxWeight policy and Blackwell's Approachability Theorem.
In Section \ref{RegretQ}, we analyse the regret of queueing systems when service rates are unknown. 
In Section \ref{information}, we discuss the use of detailed state information and its impact on performance. 
In Section \ref{sec:RL}, we review the recent progress applying reinforcement learning theory to queueing systems. 

%
%%
%Also although some problems for queueing systems can be directly reduced prior to estimation. There are other issues that mean that theoretical results from areaas such as reinforcement learning cannot be directly applied to queueing systems without a positing potentially and unrealistic assumptions.
%
%
%The analysis of queueing system is this essential to many areas of operations research, such as inventory control, communication systems and manufacturing service operations healthcare. The theory of queueing is traditionally modelling discipline. Here parameters of the model are often assumed in advance and the implications of these assumptions on Q size loss reneging are analysed.
%%
%Less attention is paid to the mechanisms by which these parameters are estimated.

\section{Blackwell Approachability and MaxWeight.}\label{BAMW}

In this section, we relate classical approaches to learning, control and queueing. Specifically, we discuss connections between Blackwell's Approachability Theorem and the MaxWeight policy. 

We start by presenting Blackwell's Approachability Theorem (Theorem \ref{Blackwell:Blackwell}). As we will discuss in more detail shortly. Blackwell's Approachability Theorem gives conditions for the time average value of a controlled random walk to approach a convex set in the setting where the random walk may be perturbed adversarially (see Section \ref{vvgame}). The result is considered to be foundational in online convex optimization, as a number of algorithms and approaches can viewed as consequence of Blackwell's result. Specifically \cite{blackwell1956analog}, proves a result on the regret of online algorithms called the Hannan-Gaddum Theorem is a corollary of his theory (Theorem \ref{Hannan}). 
The regret of an online algorithm is difference between the reward of the algorithm and the reward best fixed parameter choice in hindsight (cf. \eqref{Regret}). 

We then introduce the MaxWeight policy and prove its stability for subcritical queueing networks. Finally, we prove that the stability of MaxWeight can be viewed as a consequence of Blackwell's Approachability Theorem.

%\cite{abernethy2011blackwell} later showed that sub-linear regret also implies approachability. 

\subsection{Blackwell Approachability}

Blackwell's Approachability Theorem is a foundational result in game theory, learning and control. The result of \cite{blackwell1956analog} gives a necessary and sufficient condition for the mean of a vector payoff to approach  a convex set under adversarial perturbations. 
The result is a generalization of the Minimax Theorem for two-person zero-sum games.
In a subsequent note, \cite{blackwell1954controlled} observed that this approachability property can be used to prove sub-linear regret of sequential decision policies. 
Below we present the model setting and then state Blackwell's Approachability Theorem; a proof of this result is provided in the appendix.

%Sequentially a player decides to play $\{d_t\}_{t=1}^\infty$ and his adversary decides $\{a_t\}_{t=1}^\infty$. 

\subsubsection{The Minimax Theorem.}
Blackwell's Approachability Theorem is an extension of the Minimax Theorem, see \cite{morgenstern1953theory}. 
So we start by briefly discussing the Minimax Theorem. The Minimax Theorem applies to two-person zero-sum games, here you make a decision $d\in \mathcal D$ and an adversary makes a decision $a\in \mathcal A$. (Here $\mathcal D$ and $\mathcal A$ are closed, bounded, convex subsets of $\mathbb R^p$.) For decision $d$ and adversary decision $a$, you receive a reward $d^\top R a$ and the adversary receives a cost of $d^\top R a$, where $R$ is a $p\times p$ real-valued matrix. Given that you want to maximize reward whiles the adversary wishes to minimize costs, 
the Minimax Theorem states that
\begin{align*}
  \min_{a \in \mathcal A} \max_{d \in \mathcal D} d^\top R a = \max_{d \in \mathcal D} \min_{a \in \mathcal A} d^\top R a\,.
\end{align*}
The value, $v$, of the expression above is called the value of the game. Notice the inequality $\min_{a \in \mathcal A} \max_{d \in \mathcal D} d^\top R a  \geq v$ states that if we know $a$ then we can make a decision that gets at least value $v$. While $\max_{d \in \mathcal D} \min_{a \in \mathcal A} d^\top R a  \geq v$ gives the seemly stronger statement that there exists a decision the regardless of $a$ we can get at least value $v$. In Blackwell's terminology this means that the decision maker can approach the set $[v,\infty)$. Blackwell's Approachability Theorem generalizes this notion to a game with a vector of rewards, and the notion of approachability to convex sets. 

\subsubsection{A vector valued game.}

\label{vvgame}
At time $t\in \mathbb N$, 
a player makes a decision $d(t) \in \mathcal D$ and its adversary makes a decision $a(t)\in \mathcal A$. (Here $\mathcal D$ is a closed, bounded, convex subset of $\mathbb R^p$ and $\mathcal A$ is closed, bounded, convex subset of $\mathbb R^p$.) From this an $n$-dimensional vector of payoff is made
\begin{align}\label{VecPayoff}
  R(d(t),a(t)) = \sum_{i=1}^p \sum^q_{j=1} d_i(t)R_{ij} a_j(t) 
\end{align}
where, for each $i$ and $j$, $R_{ij}\in \mathbb R^n$. We let $R_{\max} \geq \max_{ij} ||R_{ij}||_2$.
Decisions from the player and adversary at each time may be a function of previous decisions and payoffs. We let $\mathcal F_{t}$ denote the decisions and payoffs up until time $t$.  We allow $d(t)$ and $a(t)$ to be random but must be conditionally independent given the past payoffs and decisions, $\mathcal F_{t-1}$. We let $\bar a(t)= \mathbb E [ a(t) | \mathcal F_{t-1}]$ and $\bar d(t)= \mathbb E [ d(t) | \mathcal F_{t-1}]$.
Starting from some initial position $Q(0) \in \mathbb R^q$, the average payoff vector is
\begin{align}\label{Qbar}
   \bar Q(t)=\frac{1}{t}\left[Q(0)+\sum_{s=1}^t
R(d(s),a(s))\right].
\end{align}
Without loss of generality, we assume that $Q(0)$ is chosen so that $||Q(0)||_2 \leq R_{\max}$. 

The task of the player is to make a sequence of decisions $\{ d(t) \}_{t=1}^\infty$ such that
the mean payoff $\bar Q(t)$ converges towards a closed, convex set $\mathcal Z \subset \mathbb R^n$, regardless of decisions of the adversary $\{ a(t) \}_{t=1}^\infty$ . We say that the set $\mathcal Z$ is \emph{approachable} if this convergence is possible regardless of the sequence of decisions made by the adversary. Specifically, for all sequences $\{a(t) \}_{t=1}^\infty$ there exists a sequence $\{ d(t) \}_{t=0}^\infty$ such that
\begin{align*}
  D(\bar Q(t) , \mathcal Z) \xrightarrow[t\rightarrow \infty]{} 0\, ,
\end{align*}
where $D(\bar Q,\mathcal Z)$ gives the $L^2$-norm of the distance between $\bar Q$ and $\mathcal Z$:
\[
D(\bar Q , \mathcal Z) := \sqrt{ \mathbb E \Big[ \min_{z \in \mathcal Z} || \bar Q - z ||^2\Big]} \, .
\]

% regardless of the decisions adversary

\subsubsection{Blackwell's Approachability Theorem.}
Blackwell's Approachability Theorem provides necessary and sufficient conditions for approachability to hold.

%, regardless of
%the adversary's decisions, the player makes the sequence of vectors $\{\bar Q_t\}_{t=1}^\infty$ approach a
%convex set ${\mathcal A}$.
%\begin{itemize}
%\item We consider a sequence $(d_1,a_1), (d_2,a_2), (d_3,a_3), ... $. Here each $d_t$ is  a
%probability distribution on a set of (pure) decisions $\{1,...,n\}$  and similarly $a_t$ is
%a probability distribution on $\{1,...,n\}$.
%\item For each pair $(p,q)$, there is a payoff vector $R(p,q)\in{\mathbb R}^k$ and average payoff vector
%$\bar Q_t\in{\mathbb R}^k$. Here 
%\begin{equation}\label{blackwell:A}
% R(p,q):=\sum_{i=1}^n\sum_{j=1}^m d_iR_{ij}a_j\quad\text{and}\quad \bar Q_t=\frac{1}{t}\sum_{\tau=1}^t
%R(d_\tau,a_\tau).
%\end{equation}
%\item At each time $t$, we assume $d_t$ is chosen as a function of $R(\cdot,\cdot)$ and the past
%decisions $\{(d_\tau,a_\tau)\}_{\tau=1}^{t-1}$. 
%\item We say that $\{\bar Q_\tau\}_{\tau=1}^\infty$ approaches a set ${\mathcal Z}\subset{\mathbb R}^p$ if the distance
%between $\bar Q_t$ and ${\mathcal Z}$ converges to zero, namely, 
%%with probability one\footnote{When we say an event $E$ occurs with probability one we mean $\bP(E)=1$.}
%\begin{equation*}
%d(\bar Q_t,{\mathcal Z}):=\inf_{\alpha\in{\mathcal Z}} ||\bar Q_t-\alpha||_2 \xrightarrow[t\rightarrow\infty]{} 0.
%\end{equation*} 
%%\item The reward function $R(x,y)$ is known and based on the past $(x_s,y_s), s=1,...,t-1$, we
%%chose a decision $x_t$ then subsequent to this an outcome, $y_t$, is chosen.
%\end{itemize}
\begin{theorem}[Blackwell's Approachability Theorem]\label{Blackwell:Blackwell} The following are
equivalent
\begin{enumerate}
\item ${\mathcal Z}$ is approachable.
\item For every $q$ there exists $p$ such that $R(p,q)\in {\mathcal Z}$.
\item Every half-space containing ${\mathcal Z}$ is approachable.%\footnote{Recall a half-space in ${\mathbb R}^k$ is a set ${\mathbb H}=\{a\in{\mathbb R}^k : n\cdot a \leq c\}$ for some $n\in{\mathbb R}^k$ and $c\in{\mathbb R}^k$.} 
\end{enumerate}
\end{theorem}

We focus on the equivalence between Parts 1 and 3, above, and, for reasons of brevity, we assume that the equivalence between Parts 1 and 2 has already been shown. We note that approachability of half-spaces can be shown to be equivalent to the classical Minimax Theorem. Further we note that Blackwell's proof that Part 3 implies Part 1 is constructive, and the algorithm applied is as follows.

We define $ P(t)$ to be the projection with respect to the Euclidean norm of $\bar Q(t)$ onto $\mathcal Z$. We let 
\begin{align}
	\hat n(t) &:= \bar Q(t)-   P(t)
\label{normal}
\\
v(t) &:= P(t) \cdot (\bar Q(t)-   P(t))\, .
\label{v}
\end{align}
This defines a hyperplane 
\begin{equation}\label{hyperplane}
\mathcal H(t) := \{ r \in \mathbb R^n : \hat n(t) \cdot r \leq v(t) \}\,.
\end{equation}
Note that $\mathcal Z$ is contained in $\mathcal H$. %Further given  Part 3, $\mathcal H$ should be approachable and thus Part 2 holds with $\mathcal H$ in-place of $\mathcal Z$.
Blackwell proposes to choose a random variable $d(t)$
\begin{equation}\label{dRa}
\mathbb E [ R( d(t),a) ]\in \mathcal H(t), \qquad \forall a \in\mathcal A \, .
\end{equation}
The existence of this choice $d(t)$ is not obvious at this point, but this will follow as a consequence of the Minimax theorem. 

\begin{figure}[hbt]
\centering
  \includegraphics[width=0.6\textwidth]{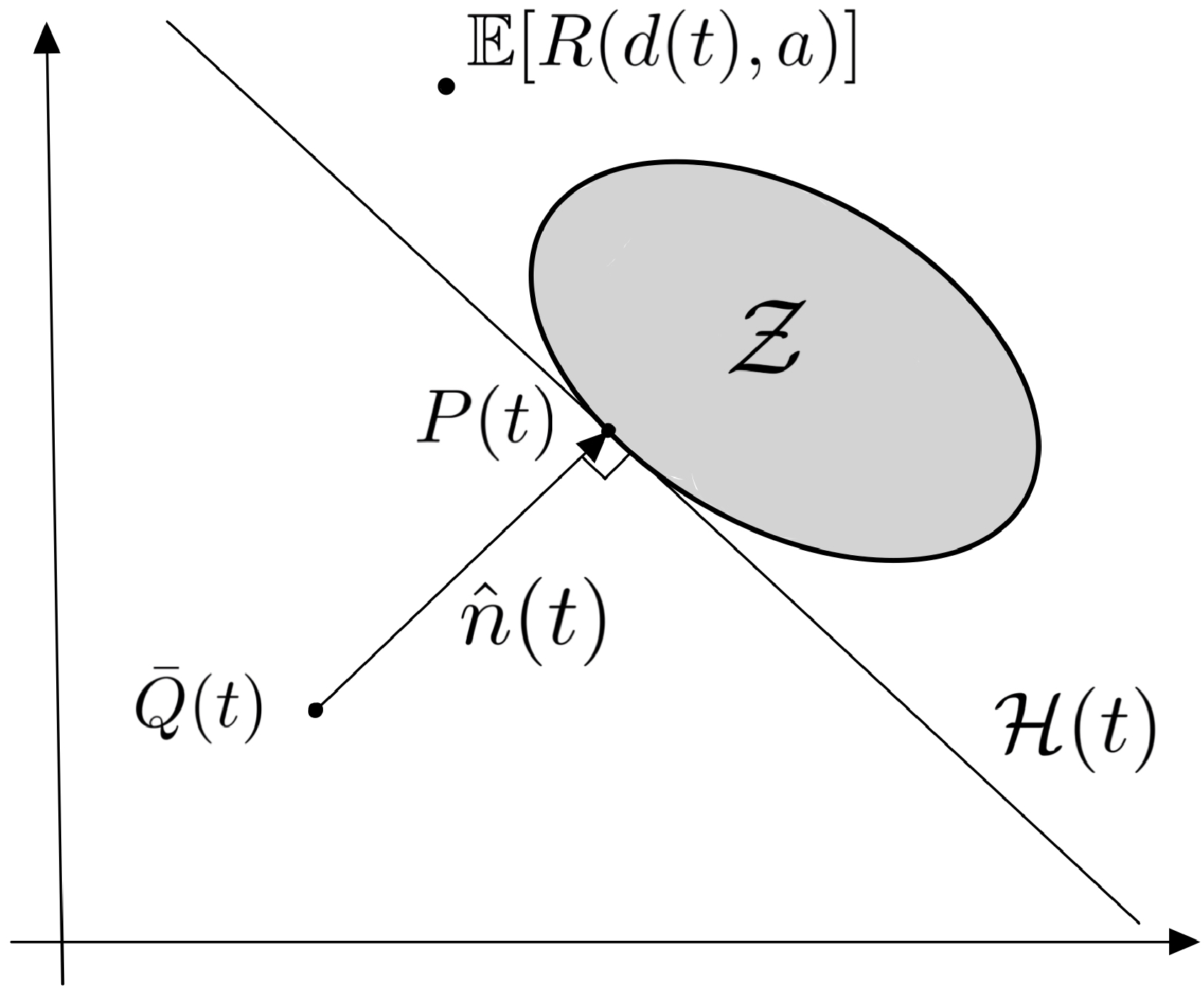}
  \caption{Blackwell's Approachability Theorem: for $\mathcal Z$ to be approachable, we require $\mathbb E [ R( d(t),a) ]$ to belong to the hyperplane $\mathcal H(t)$ for each $a\in \mathcal A$. Here $\mathcal H(t)$ is the hyperplane constructed from projecting $\bar Q(t)$ onto $\mathcal Z$.}
\end{figure}

%\noindent \textbf{Remark.} The vectors $a(t)$ and $d(t)$ are modeled as deterministic above. However it is possible to consider $a(t)$ and $d(t)$ to be bounded random variables which are conditionally independent given their past values. Here we can apply Blackwell approachability with respect to the mean vectors $\bar a(t)$ and $\bar d(t)$ and then apply Azuma-Hoeffding to $a_i(t)d_j(t)-\bar a_i(t)\bar d_j(t) $ to establish a bound in place of \eqref{blackwell: conv} which is of the order $O(\log T/\sqrt{T})$.

\subsubsection{Sublinear Regret as an application of Blackwell Approachability.}\label{SubLinReg}
The regret minimization framework is by now standard for the analysis of online learning algorithms, online convex optimization and sequential decision making. 
Again here there are decisions that can be made over time and that receive rewards (or losses). 
The regret, which we define below, is measure the performance of a sequence of decisions against the best fixed decision. Literally, how much we regret not making the best decision in hindsight. 
A ``good" algorithm should have regret that grows sub-linearly with time.
This essentially implies that the algorithm is able asymptotically learn the best decision. 
The existence of sub-linear regret algorithms is a consequence of approachability, see \cite{blackwell1954controlled}. This connects Blackwell's Approachability Theorem with a rich stream of literature that encompasses bandit algorithms (see \cite{lattimore2020bandit}) and modern neural network optimizers like AdaGrad (see \cite{hazan2019introduction}). 

To review Blackwell's result here, we consider the following setting:
again, a player makes a decision $d(t) \in \mathcal D$ and its adversary makes a decision $a(t)\in \mathcal A$. The decision set $\mathcal D$ is assumed to be the set of probability distributions on actions  indexed by $i=1,...,p$. The player receives a reward
 $r(i,a(t))$ for choosing index $i=1,...,p$. Consequently, the reward for distribution $d(t)$ is \[r(d(t),a(t)) :=\sum_{i=1}^p d_i(t) r(i,a(t)).\] 
The regret of the sequence $\{ d(t) \}_{t=1}^T$ is 
\begin{align}\label{Regret}
  \mathcal R\!g(T) := \max_{i=1,...,p}\sum_{t=1}^T r(i,a(t)) - \sum_{t=1}^T r(d(t),a(t))   \, .
\end{align}
%In this setting consider an 
A good algorithm should have low regret regardless of the sequence $\{ a(t) \}_{t=1}^T$. Notice that if $\mathcal R\!g(T) \leq 0$ holds, then the performance of the algorithm is at least as good as the best fixed choice in $\mathcal D$. Further if $\limsup_{T\rightarrow\infty} \mathcal R\!g(T)/T \leq 0$ then we can see that the algorithm is in effect learning the best action. Such algorithms are referred to as being Hannan consistent, see \cite{cesa2006prediction}.

 One interesting consequence of Blackwell's Approachability Theorem is that it can be used to construct an algorithm that is Hannan consistent.
\begin{theorem}[Hannan-Gaddum Theorem]
\label{Hannan}
%Suppose $R(p,q)$ as defined in \eqref{blackwell:A} is such that $R(p,q)\in{\mathbb R}$. 
%As before, we consider
%an adversary where the sequence $\{q_t\}_{t=1}^\infty$ is either Arbitrary[D[D[D[D[D[D[D[a and
%predetrimned, or where $q_{t+1}$ is a function of the player's choices $p_1,...,p_t$.
There exists a playing strategy $\{d_t\}_{t=1}^\infty$ such that for any $\{a_t\}_{t=1}^\infty$
\begin{equation}\label{blackwell:regret}
\limsup_{T \rightarrow \infty}\; \mathbb E \left[ \frac{\mathcal R\!g(T)}{T} \right] \leq 0. 
\end{equation}
In other words, our performance in the game is asymptotically as good as the best fixed action.
\end{theorem}

The result above shows that regardless of the choices of the adversary,  there are always ways to learn the best fixed decision. A formal proof of this result is given in the appendix. However, the result holds essentially as follows. The regret will go to zero if each component of the following regret vector approaches zero. Specifically, if we define
\begin{equation*}
\textstyle
	\overset{\rightarrow}{\mathcal R\! g}(T) = \left( \sum_{t=1}^T r(i,a(t)) - \sum_{t=1}^T r(d(t),a(t)) : i = 1,...,p\right)
\end{equation*}
and $\mathcal Z$ is the negative orthant. I.e. 
$
  D(\overset{\rightarrow}{\mathcal R\! g}(T), \mathcal Z) \rightarrow 0
$ as $T\rightarrow \infty$.
For this problem, $R(d,a) = (r(i,a) - r(d,a): i=1,...,p)$.
Notice that for any decision by the adversary $a$ we can chose a decision $d$ so that $R(d,a)\in \mathcal Z$. This verifies condition (2) in Blackwell's Approachability Theorem and so the result holds. See Figure \ref{fig:test} for an illustration of this proof.

%The proof follows as a consequence of Blackwell's Approachability Theorem.

%\begin{figure}[hbt]
%\centering
%  \includegraphics[width=0.6\textwidth]{Pictures/Regret_Fig}
%  \caption{Approachability for Regret: Here $\mathcal Z$ is the negative orthnant. Since for each $a$ there is a decision $d$ (specifically, the maximizer of $r(d,a)$) such that reward difference belongs to $\mathcal Z$, the set $\mathcal Z$ is approachable.}
%\end{figure}
%
%\begin{figure}[hbt]
%\centering
%  \includegraphics[width=0.6\textwidth]{Pictures/MaxWeight_Fig}
%  \caption{Approachability for Regret: Here $\mathcal Z$ is the negative orthnant. Since for each $a$ there is a decision $d$ (specifically, the maximizer of $r(d,a)$) such that reward difference belongs to $\mathcal Z$, the set $\mathcal Z$ is approachable.}
%\end{figure}

\begin{figure}
\centering
\begin{subfigure}{.5\textwidth}
  \centering
  \includegraphics[width=1.02\linewidth]{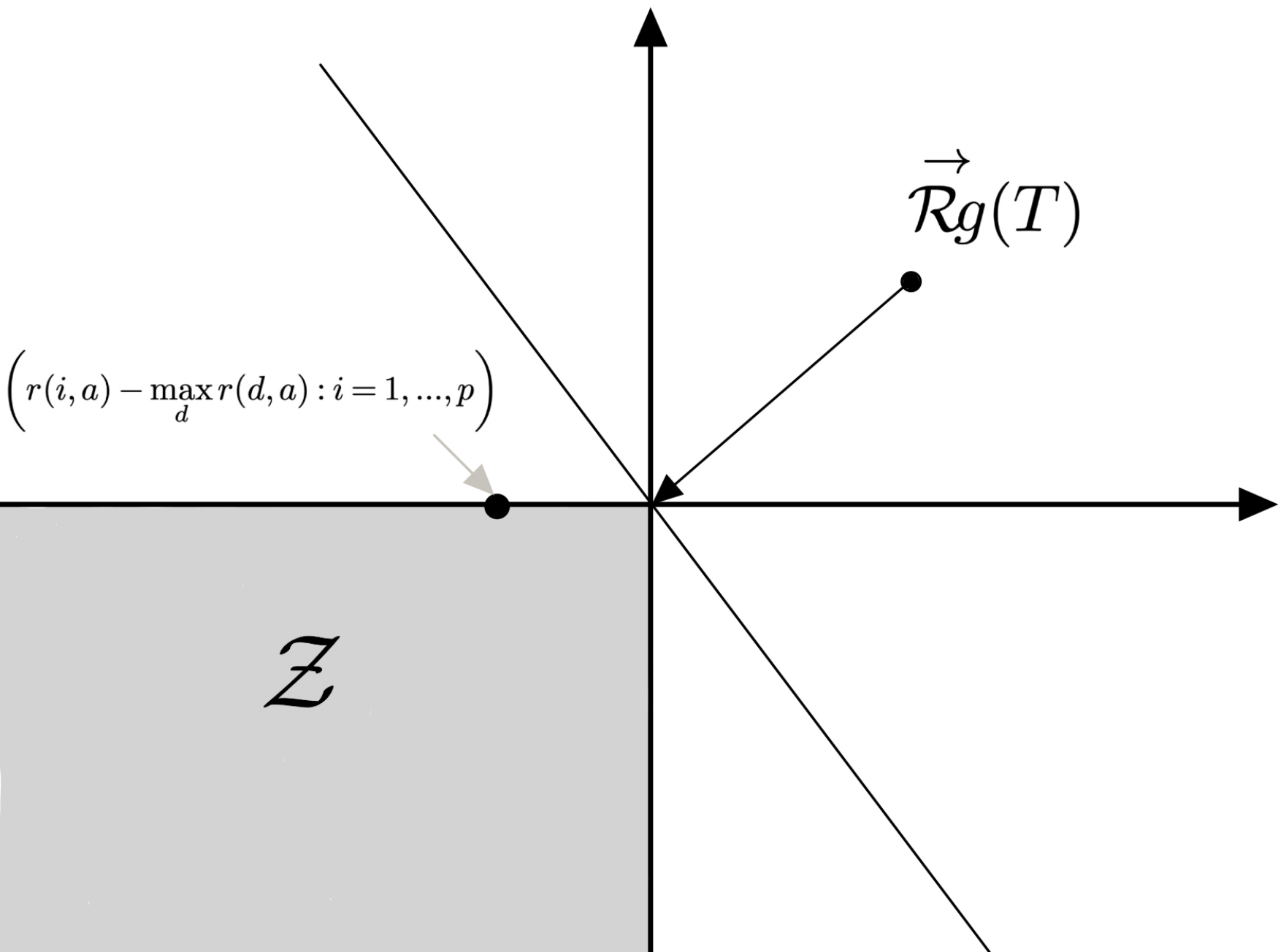}
  \caption{Approachability for Regret}
  \label{fig:sub1}
\end{subfigure}%
\begin{subfigure}{.5\textwidth}
  \centering
  \includegraphics[width=.88\linewidth]{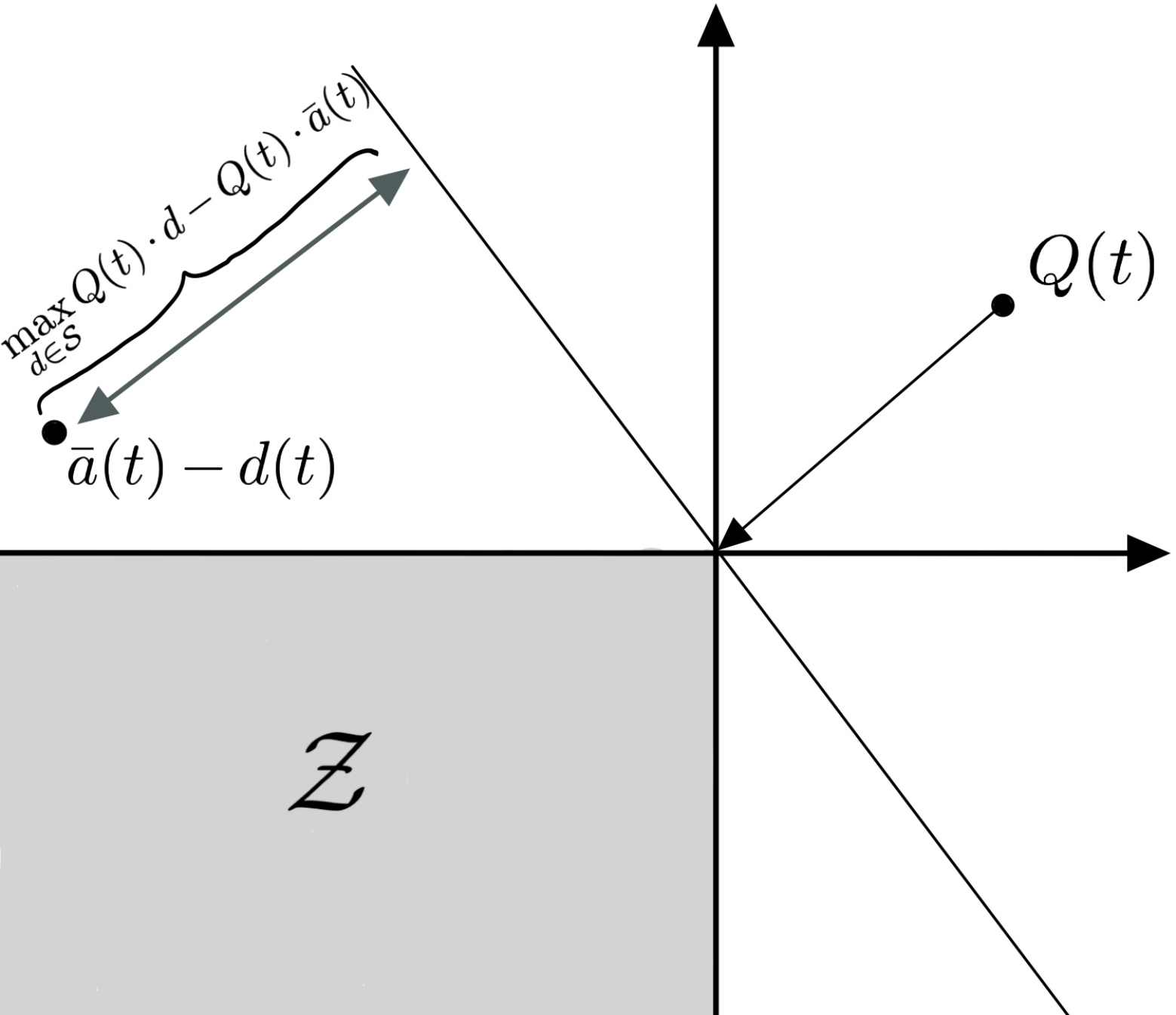}
  \caption{Approachability for Queues}
  \label{fig:sub2}
\end{subfigure}
\caption{(a) Approachability for Regret : Here $\mathcal Z$ is the negative orthant. Since for each $a$ there is a decision $d$ (specifically, the maximizer of $r(d,a)$) such that reward difference belongs to $\mathcal Z$. Thus, the set $\mathcal Z$ is approachable.
(b) Approachability for Queues: Since for each $a$ there is a decision $d$ (specifically, the maximizer of $Q(t)\cdot d$) such that reward difference belongs is to the right of the hyperplane with normal $Q(t)$.}
\label{fig:test}
\end{figure}
%
%\begin{figure}
%\centering
%\begin{minipage}{.5\textwidth}
%  \flushleft  \includegraphics[width=1.\linewidth]{Pictures/Regret_Fig}
%  \captionof{figure}{Approachability for Regret: Here $\mathcal Z$ is the negative orthant. Since for each $a$ there is a decision $d$ (specifically, the maximizer of $r(d,a)$) such that reward difference belongs to $\mathcal Z$. Thus, the set $\mathcal Z$ is approachable.}
%  \label{fig:test1}
%\end{minipage}%
%\begin{minipage}{.50\textwidth}
%  \flushright
%  \includegraphics[width=.88\linewidth]{Pictures/MaxWeight_Fig}
%  \captionof{figure}{Approachability for Queues: Since for each $a$ there is a decision $d$ (specifically, the maximizer of $Q(t)\cdot d$) such that reward difference belongs is to the right of the hyperplane with normal $Q(t)$. }
%  \label{fig:test2}
%\end{minipage}
%\end{figure}

\subsection{Stability and Approachability in MaxWeight} 

A classical approach to scheduling in queueing networks is the MaxWeight algorithm, which was first introduced in the seminal work of \cite{tassiulas1990stability}. In that paper, it is proved that the MaxWeight scheduling algorithm has the property of stabilizing a queueing network whenever it is possible to construct a stabilizing policy -- a statement that is somewhat similar to the statement of Blackwell's result. Moreover, \cite{neely2008fairness} extends the framework to the optimization of different utility function objectives and \cite{neely2010universal} extends the framework to adversarial arrivals.
Below we discuss how MaxWeight can be interpreted as special case of Blackwell's policy \eqref{dRa}. Thus the desirable stability properties of MaxWeight can be seen as a consequence of approachability.

Below, we present the queueing network setting and the MaxWeight policy. Afterwards we present the main stability result for MaxWeight; a proof of the result is provided in the appendix. Following this, we make the connection with Blackwell's Approachability Theorem. 

\subsubsection{Switched Queueing Network and MaxWeight.}\label{subsec:switch}
We consider a discrete time queueing network, called a switched queueing network. In this model, there are constraints on which queues can be served simultaneously. 
%Our task is to serve queues in a way that stabilizes the network. 
%In this setting MaxWeight a very well-known policy \cite{tassiulas1990stability}. MaxWeight has the desirable property that it will stabilize the queueing network essentially whenever it is possible to stabilize the network. We show below that the MaxWeight policy is an example Blackwell's policy and so desirable stability properties associated with MaxWeight can viewed as an application of Blackwell's Approachability Theorem. This is very much analogous to the way no-regret learning can be seen as an application of Blackwell Approachability. 
%
%Below, we describe: the switched queueing network model; the MaxWeight policy along with its stability properties; and then we establish its connection with Blackwell Approachability. 
We let $j=1,...,q$ index the set of queues. At time $t \in\mathbb Z_+$, the number of arrivals to each queue is given by the vector $a(t) = (a_j(t): j =1,...,q) \in \mathbb Z_+^q$. For simplicity, we suppose that the components above are bounded by some maximum value $a_{\max}$. The vector $a(t)$ may be random, in which case, we let $\bar a(t)=(\bar a_j(t): j = 1,...,q)$ be the expectation of $a(t)$.
We let $d(t) = (d_j(t) : j=1,...,q)$ be the number of departures for each queue at time $t$.
We suppose that $d(t)$ belongs to some set of schedules $\mathcal S \subset \{0,1\}^q$. \footnote{Many of the arguments discussed below likely extend to the case $\mathcal S \subset Z_+^q$. However, the assumption $\mathcal S \subset \{0,1\}^q$ simplifies our arguments when a scheduling decision empties a queue.}
Again we suppose that the elements of $d(t)$ are bounded by some maximum value $d_{\max}$ and that each queue can be served by some element $d$ in $\mathcal S$. Further, we suppose that the set of schedules $\mathcal S$ is monotone, meaning that for any $\sigma \in\mathcal S$ and $\sigma'\in \mathbb Z_+^q$ if  $\sigma' \leq \sigma$ component-wise then $\sigma' \in \mathcal S$. We let $<\mathcal S>$ be the convex closure of $\mathcal S$ and we let $<\mathcal S>^\circ$ be the interior of $<\mathcal S>$. 

From some initial queue size vector $Q(0) = (Q_j(0) : j=1,..,q)$, we can define the queueing process %$Q_j(t)$, $t\in\mathbb Z_+$ and $j=1,...,q$., where for some $Q_j(t)$ we define
\begin{align}\label{Qdynamic}
  Q_j(t+1) = Q_j(t) + a_j(t) - d_j(t), \qquad j=1,...,q\, ,
\end{align}
for $t \in \mathbb Z_+$.
(Note by our monotonicity property on $\mathcal S$, we may assume that $d_j(t)=0$ whenever $Q_j(t)=0$ and more generally that $d_j(t) \leq Q_j(t)$ so that queues remain positive.)

The MaxWeight policy is a well-known scheduling rule for choosing the departure vector $d(t)$ among the set of schedules $\mathcal S$. Specifically, at time $t$, MaxWeight chooses 
\begin{align}\label{MW}
  d(t) \in \argmax_{\substack{\sigma \in \mathcal S}} \;\;\ \sum_{j=1}^q Q_j(t) \sigma_j \, . 
\end{align}
We assume that if a queue is empty, i.e. $Q_j(t)=0$, then $d(t)$ is chosen above so that $d_j(t) = 0$. Any other ties may be broken arbitrarily.
Note that queues always remain non-negative under MaxWeight. Since the scheduling decision is a function of the current queues sizes, the resulting queue size process under the recursion \eqref{Qdynamic} is a Markov chain. Thus we can discuss the positive recurrence and transience of this Markov chain. When the queueing network is positive recurrent then we often refer to the network being stable, whereas a transient network is referred to as unstable.

\begin{figure}
\centering
\begin{subfigure}{.5\textwidth}
  \centering
  \includegraphics[height=0.95\linewidth]{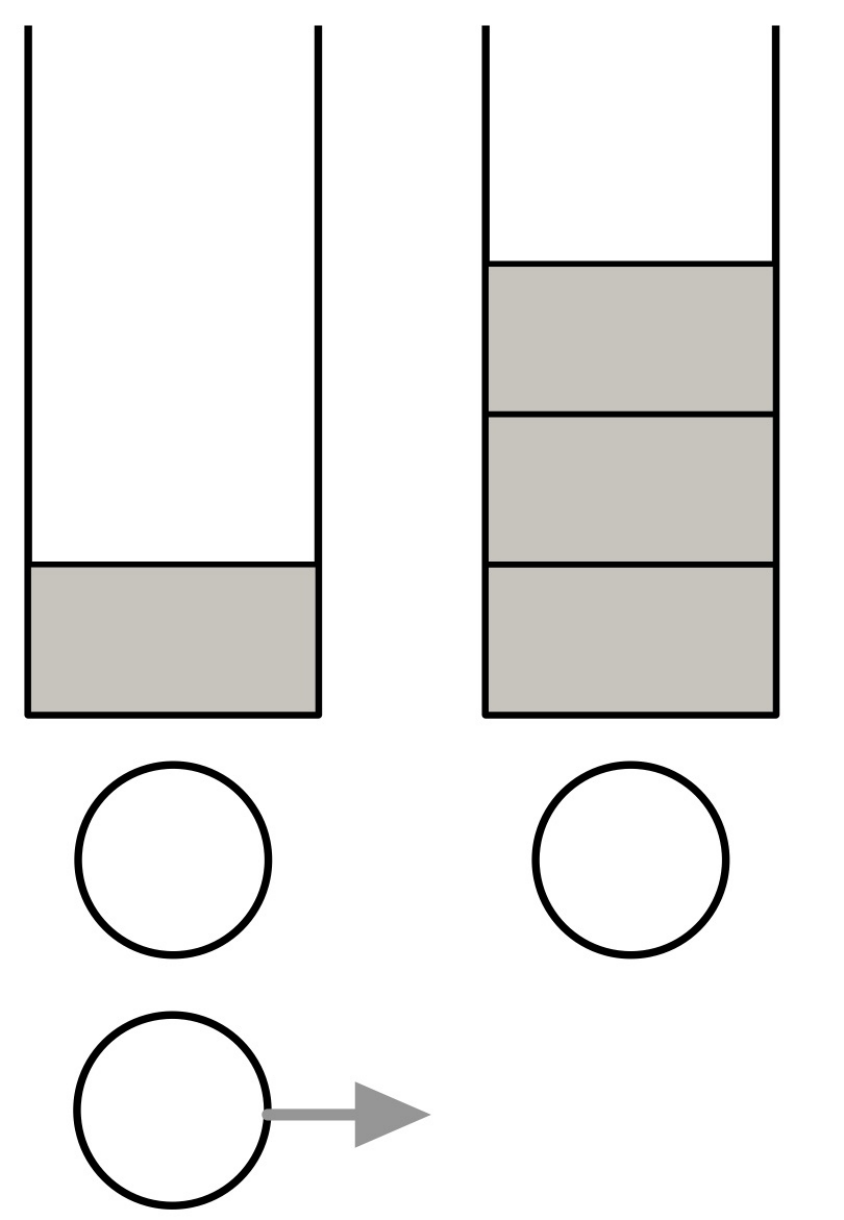}
  \caption{Two queues with an additional flexible server}
  \label{fig:sub1}
\end{subfigure}%
\begin{subfigure}{.5\textwidth}
  \centering
  \includegraphics[height=.95\linewidth]{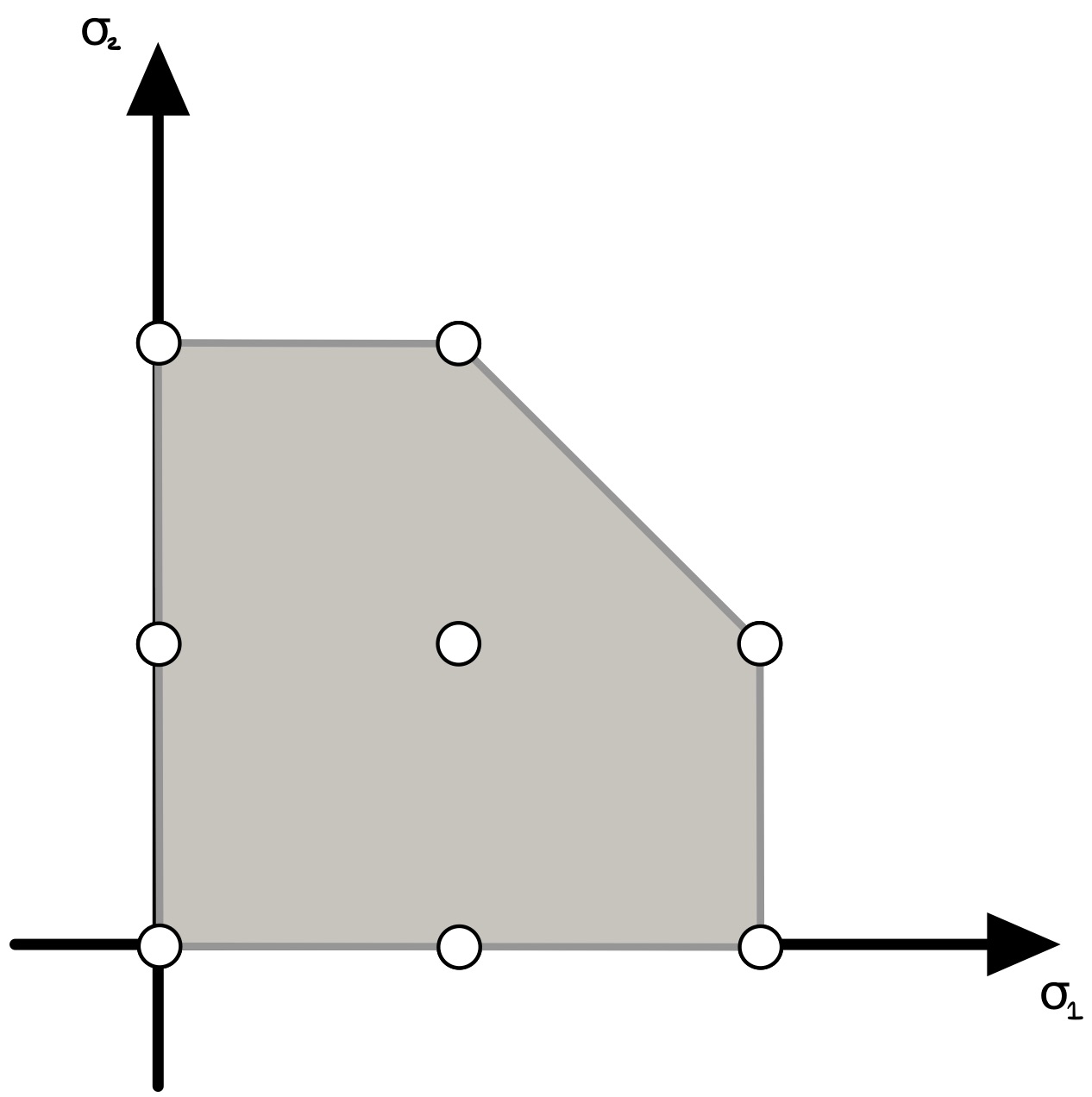}
  \caption{Scheduling set and stability region}
  \label{fig:sub2}
\end{subfigure}
\caption{(a) Switched queueing network example: each queue has one devoted server and there is an additional server that can provide one additional unit of service to one queue only.
(b) Stability Region example: the schedules for the switch network in (a) are circled; the stability region, which is the convex combination of schedules, is shaded.}
\label{fig:test}
\end{figure}

\subsubsection{Maximal Stability of MaxWeight.}

A well-known result due to Tassiulas and Ephremedes is the following

\begin{theorem}\label{MW:thrm}
Given $a(t), t\in\mathbb Z_+$ are independent identically distributed with mean $\bar a$ then
\begin{enumerate}[i)]
	\item If $\bar{a}\notin <
{\mathcal S} >$ then, regardless of the policy used, $(Q_j(t) : j=1,...,q)$ is transient.
\item  If $\bar{a}\in <
{\mathcal S} >^\circ$ then, under the MaxWeight policy, $(Q_j(t) : j=1,...,q)$ is positive recurrent.
\end{enumerate}
\end{theorem}

The first part of the theorem, above, establishes that $<\mathcal S>$ is the \emph{maximal stability region} for the switched queueing network. That is the set $<\mathcal S>$ is the largest set of arrival rates for which we can hope to find a policy to stabilize the queueing network. (Note that, in principle, this policy may need to know the arrival rates to stabilize the network.) What the second part shows is that there is a single policy, specifically MaxWeight, that achieves the maximal stability region, regardless of the which arrival rate it is inside the stability region.\footnote{Note here we are ignoring arrival rates on the boundary $\partial\!\! <\!\! \mathcal S\!\! > = <\!\!\mathcal S\!\!>\! \backslash\! <\!\!\mathcal S\!\!>^\circ$} We provide a sketch of part ii) below:

\Beginproof{(Sketch Proof of Theorem \ref{MW:thrm}ii).)}
	The change in service queue size is given by the number of arrivals minus departures. We can thus approximate this with the following ordinary differential equation:
\begin{align}\label{o.d.e.}
  \frac{d Q_j}{dt} = \bar a_j -\sigma_j(t), \qquad j =1,...,q, \quad t \in \mathbb R.
\end{align}
Before analyzing this o.d.e., let's briefly introduce some notation. We let $Q^\Sigma$ be the sum of all queue sizes and we let $P_j$ be the proportion of jobs in queue $j$, so $Q_j =Q^\Sigma P_j$. Also since $\bar{a}$ belongs to the interior of $<\!
{\mathcal S}\! >$ there exists some  $\epsilon>0$ such that 
\begin{equation}\label{abar}
	\bar{a} + \epsilon \bm 1 \in <\!
{\mathcal S}\! >.
\end{equation}
%Given $\epsilon$ we let $<\! {\mathcal S}\! >^\epsilon$ be the set arrivals $\bar a$ satisfying \eqref{abar}. 

Now, for the above o.d.e. model, we can assume that the system is stable if the (euclidean) distance of the queue sizes goes to zero. We can analyze the change in this distance as follows
\begin{align*}
   \frac{d \|Q\|^2}{dt}  
& =
	2 \sum_{j=1}^p Q_j \frac{d Q_j}{dt} 
\tag{Chain rule}
\\
&
=
2 \sum_{j=1}^p Q_j \Big( \bar a_j -\sigma_j(t) \Big)
\tag{By \eqref{o.d.e.}}
\\
&
=
2 \min_{\sigma \in <\mathcal S>} \sum_{j=1}^p Q_j \Big( \bar a_j -\sigma_j \Big)
\tag{MaxWeight Definition \eqref{MW}}
\\
&
=
2 Q^\Sigma \min_{\sigma \in <\mathcal S>} \sum_{j=1}^p P_j \Big( \bar a_j -\sigma_j \Big)
\\
&
\leq 
2 Q^\Sigma \max_{P \in \mathcal P} \min_{\sigma \in <\mathcal S>} \sum_{j=1}^p P_j \Big( \bar a_j -\sigma_j \Big)
\\
&
=
2 Q^\Sigma \min_{\sigma \in <\mathcal S>} \max_{P \in \mathcal P}  \sum_{j=1}^p P_j \Big( \bar a_j -\sigma_j \Big)
\tag{Minimax Theorem}
\\
&
=
2 Q^\Sigma \min_{\sigma \in <\mathcal S>} \max_{j=1,...,p}  \Big( \bar a_j -\sigma_j \Big)
\\
&
\leq 
-2 Q^\Sigma \epsilon
\tag{by \eqref{abar}}
\end{align*}
From this we see that so long as the total queue size is positive then the (euclidean) distance of the queue size vector is decreasing, and so must decrease to zero. \hfill
\Endproof

The formal proof of Theorem \ref{MW:thrm} is given in the appendix. We note that the proof for MaxWeight has several connections with the proof of the Blackwell Approachability Theorem:
\begin{enumerate}
	\item The argument is constructive and applies a quadratic Lyapunov function.
\item The proof applies the Minimax Theorem.\footnote{We note that we apply the Minimax Theorem in the way that we construct our proof but this is not necessarily standard.}
\item The divergence of the policy follows from a separating hyperplane argument. 
\end{enumerate}
\noindent All of these properties indicate a close connection between MaxWeight and Blackwell Approachability, which we develop next. 

\subsubsection{MaxWeight and Approachability.}
Here we show that MaxWeight is an instance of Blackwell's policy, \eqref{dRa}, for a specific choice of $R$. 
Thus a reason for the robust stability properties of MaxWeight can be seen as a consequence of robustness to adversarial perturbations found in Blackwell's Approachability Theorem.

%Below we show that a fluid stability property of MaxWeight follows as a direct consequence of Blackwell's proof. 

\begin{proposition}\label{MW:Prop}
The MaxWeight policy \eqref{MW} is an instance of Blackwell's policy \eqref{dRa} and as a consequence of Blackwell's Approachability Theorem it follows that for any subcritical arrival rate $\bar{a}\in <
{\mathcal S} >^\circ$
\begin{align}\label{MW:PropEq}
 \mathbb E \left[ \frac{1}{T}  \sum_{j=1}^q Q_j(T) \right] \xrightarrow[T\rightarrow \infty]{}0 \,.
\end{align}
\end{proposition}

\Beginproof
%\Neil{Preamble for the proof.}
For the proof, we construct a payoff matrices $R^k$, $k=1,...,q$ was well as decisions $d'(t)$ and adversary decisions $a'(t)$ that give queueing behavior. 
We then show that for these choices the policy of Blackwell corresponds to the MaxWeight policy. 

	We define the $2q$-dimensional vectors 
\begin{align*}
  a'(t) &= (a_1(t),...,a_q(t),1,...,1) \\
d'(t) &= (1,...,1,d_1(t),...,d_q(t))
\end{align*}
where we assume that $a(t) = (a_1(t),...,a_q(t))$ is a random variable with mean $\bar a (t)$ belonging to $<\mathcal S>$ and we assume $d(t) = (d_1(t),...,d_q(t))$ belongs to $\mathcal S$. 
For each $k=1,...,q$, we define the $2q\times 2q$ matrix
\begin{align*}
  R^k = 
\begin{pmatrix}
	\delta_k & 0 \\
0 & -\delta_k
\end{pmatrix}
\end{align*}
where $\delta_k$ is the $q\times q$ diagonal matrix with only one non-zero diagonal entry in the $k$-th entry which is set equal to one.
Notice that
\begin{align*}
  d'(t) R^k a'(t) = a_k(t) - d_k(t),
\end{align*}
for $k = 1,...,q$. Thus using the notation of \eqref{VecPayoff}, we have that 
\begin{align*}
  R(d'(t),a'(t)) = a(t) - d(t)\,.
\end{align*}
So, the vector payoff $R(d'(t),a'(t))$ is exactly the change in queue length given in \eqref{Qdynamic}. 
So $\bar Q(T)$ as given in \eqref{Qbar} is the time averaged queue size
\begin{align*}
  \bar Q(T) =  \frac{1}{T}\left[Q(0)+\sum_{t=1}^T
\{ a(t) -d(t) \}\right] =  \frac{Q(T)}{T}\, .
\end{align*}
This shows that we can consider a switched queueing network within the framework of Blackwell.

We next consider the approachability properties of the switched queueing network. Specifically we consider approachability of the set 
\[
\mathcal Z = \{ z\in \mathbb R^q  : z_k \leq 0,\;\; \forall k=1,...,q \}.
\]
Clearly the projection onto this set, $P(t)$, equals zero for any $\bar Q(t) \geq 0$. So the hyperplane $\mathcal H(t)$, \eqref{hyperplane}, has normal vector 
$
  \hat n(t) = \bar Q(t)
$
 and so is given by
\begin{align*}
  \mathcal H(t) = \{ r\in \mathbb R^q : \bar Q \cdot r \leq 0 \}\, .
\end{align*}
Recall the from Blackwell's policy approachability holds when 
\[
R(d(t),a) \in \mathcal H(t),\,\quad \forall a \in <\mathcal S>\, .
\] 
Thus given the definition of $\mathcal H(t)$ and $R(d',a')$, above, we arrive at the condition
\begin{align*}
  \bar Q(t) \cdot d(t) \geq \bar Q(t) \cdot a\, , \qquad \forall a \in <\mathcal S> \,,
\end{align*}
or equivalently (maximizing over $a \in <\mathcal S>$ and noting that $\bar Q(t) = Q(t)/t$)
\begin{align*}
   Q(t) \cdot d(t) \geq  \max_{a \in <\mathcal S>}  Q(t) \cdot a = \max_{\substack{\sigma \in \mathcal S}} \;\;\ \sum_{j=1}^q Q_j(t) \sigma_j
\end{align*}
(The equality above holds since a maximizer of a linear function over a polytope must lie on an (extreme) point in $\mathcal S$.)
Thus by choosing $d(t)$ according to the MaxWeight policy, we maximize the lefthand side which in turn gives us a policy that satisfies the criteria of Blackwell's Approachability Theorem, and thus as a direct consequeunce we have that
\begin{align*}
 || \bar Q(T) ||_{L_2} \xrightarrow[T\rightarrow \infty]{} 0\end{align*}
since  $|| \bar Q(T) ||_{L_1} \leq || \bar Q(T) ||_{L_2} $, the result \eqref{MW:PropEq} holds. 
\Endproof

\noindent A few remarks about the above result can be made. 

\begin{enumerate}
\item The rate of convergence can be found to be of the order $O({1}/{\sqrt{T}})$. Thus the result can be seen as a form of ``fluid stability" or ``rate stability" for the policy.  Because the rate of departures from the network is of the same order as the rate of arrivals. Provided $a \in < \mathcal S >^\circ$ the result can be improved to form a bound that is $O({1}/{{T}})$. 
\item The result proven above holds when the arrival rates can change over time (in an adversarial way) so long as the expected number of arrivals at each time belongs to the interior of the stability region $ <{\mathcal S} >^\circ$. 
\item We note that there are variants of Blackwell Approachability that expand upon the basic version to incorporate general Lyapunov functions (cf.~\cite{cesa2006prediction}).  If we apply these alternative algorithms for queueing network scheduling then we arrive the $f$-MaxWeight policies. This is a variant of MaxWeight where the queue size $Q_j$ in the MaxWeight policy is replaced by $f(Q_j)$ for a suitably positive increasing function $f$.   
\item The argument given in Proposition \ref{MW:Prop}, can be extended to the BackPressure policies. The BackPressure policy is an extension of the MaxWeight policy introduced by \cite{tassiulas1990stability}.  BackPressure allows for routing decisions that send jobs through multiple queues. BackPressure can also be seen to be an instance of Blackwell Approachability for a suitable choice of $R$.
\item MaxWeight as stated assumes unit sized jobs and there is no randomness in the service received by a job. However, the policy can be extended to jobs which have a random size or a random probability of being served. In this case, maximal stability does not hold under the policy \eqref{MW}. We need to consider a policy of the following form
\begin{align}\label{MW}
  d(t) = \argmax_{\substack{\sigma \in \mathcal S }} \;\;\ \sum_{j=1}^q \mu_j Q_j(t) \sigma_j \, ,
\end{align}
where $\mu^{-1}_j$ is the mean size of jobs at queue $j$. This can be seen as a consequence of the form of the function $R(d,a)$ in this case. Here service must be given within each matrix $R^k$ in order to represent this model within the Blackwell framework, specifically,
\begin{equation*}
 \mathbb E[ R^k] = 
\begin{pmatrix}
	\delta_k & 0 \\
0 & -\mu_k \delta_k
\end{pmatrix}.
\end{equation*}
However arrivals can be implemented as vector applied to these matrices.
So arrivals can be specified by an adversary but service must be specified as a parameter of the game. For this reason, the MaxWeight policy is robust to adversarial perturbations in arrivals, but may not have the same property when we consider adversarial job sizes. 
Nonetheless the mean job sizes and other statistics can still be learned. We discuss this in Section \ref{RegretQ}. 
%This gives some rationale for the requirement to change the policy under random services. 
\end{enumerate}

\subsection{Literature Review}

%\cite{5717885} notes that MaxWeight is robust to time varying demand. 

As mentioned earlier, the original  Approachability Theorem is due to \cite{blackwell1956analog}. 
The connection with sub-linear regret is described by \cite{blackwell1954controlled}. 
The use of Blackwell's condition is now commonly applied to analyze the regret of online learning algorithms. See \cite{cesa2006prediction} for an overview. The text of \cite{cesa2006prediction} provides a generalized Lyapunov version of Blackwell's condition, which as discussed specializes to well known variants of $f$-MaxWeight in the case of switched queueing networks, see \cite{shah2012switched}.  
It has also been shown that sub-linear regret is equivalent to Blackwell approachability, \cite{abernethy2011blackwell}.

The BackPressure policy, of which MaxWeight is a special case, was first introduced by \cite{tassiulas1990stability}.  During this period several commonly used queueing policies were found to be unstable when each queue was nominally underloaded, for instance see \cite{rybko1992ergodicity}. 
The BackPressure policy is a popular algorithm that circumvents these issues. 
It was later found that input-queued switches, which are the routers used in the core of the internet, could have a sub-optimal stability region under the greedy work-conserving strategies used at the time. For this reason the MaxWeight algorithm was proposed by \cite{mckeown1999achieving} as a strategy to achieve maximal stability.
Since then 
the MaxWeight policy has been applied extensively in the analysis of input queued switches. More generally  the MaxWeight and BackPressure policies have been widely analyzed in a broad set of queueing models.
The paper \cite{dai2005maximum} gives an account of their broad applicability in a wide variety of contexts including manufacturing, call centres as well as communications systems. 
 We refer to \cite{georgiadis2006resource} for a survey on MaxWeight and BackPressure. Also the recent text \cite{dai2020processing} provides an excellent introduction to stochastic networks, their stability as well as providing a good account of the various use cases of MaxWeight and BackPressure policy. 

The connection between Blackwell Approachability and MaxWeight/Backpressure appears to be new but we note it is a natural consequence of the two approaches. %which may well have been observed previously. %Similar observations of the robustness of MaxWeight have been made previously. 
Specifically,
the results of \cite{neely2010universal} are particularly relevant to our analysis as it is shown that MaxWeight and Backpressure are robust to changing arrival rates so long as the mean arrival rate lies within the stability set over each time period (see also \cite{neely2010stochastic}). 
The literature on adversarial queueing models is also relevant  \citep{borodin1996adversarial}.
Similar fluid stability results are sufficient for various adversarial queueing models \citep{gamarnik1998stability}.
MaxWeight has been analysed in the context of adversarial queueing by \cite{lim2013stability}. %We will discuss advertarial queueing in more detail shortly.

From a theoretical perspective, understanding of MaxWeight has improved through a series of papers analyzing state-space collapse of MaxWeight in heavy traffic, see for instance \cite{stolyar2004maxweight} and \cite{maguluri2016heavy}.
Also it can was proven that MaxWeight does not enjoy the same robust stability properties that BackPressure enjoys in large networks \cite{bramson2019stability}. 
An intriguing fluctuation bound for MaxWeight is found in \cite{sharifnassab2020fluctuation}. More recent literature on MaxWeight and BackPressure in queueing networks has focused closed systems, such as using it as a decision rule in ride sharing platforms (see \cite{banerjee2018dynamic} and \cite{kanoria2019blind}). Another application concerns road traffic signal control   \citep{le2015decentralized, varaiya2013max, xiao2015throughput}, which we will discuss in some more detail in Section \ref{sec:RL}.

A variety of recent papers combine elements of learning and MaxWeight. One area of research employs MaxWeight-like policies in designing learning and information processing systems, such as those encountered in crowd-sourcing, in order to achieve maximum throughput or minimize regret \citep{hsu2021integrated, massoulie2018capacity, shah2020adaptive}. The novelty here is that instead of working with queues of jobs, the decision maker is trying to minimize certain information backlog  that corresponds to ``uncertainty'' in the system. 
Another line of work investigates how to  augment MaxWeight with explicit learning and side information to improve performance \citep{krishnasamy2018augmenting, neely2012max}. However, they tend to focus more on parameter estimation, rather than approachabilitiy proporties for the algorithm. We will focus on such methods in the next section.

\section{Regret Bounds and Queueing}\label{RegretQ}

When we analyzed MaxWeight, we saw that stability is achievable under an unknown arrival rate. 
However, we do require knowledge of service rates to make decisions.
So if service rates are not known in advance then they must be learned. 
Recently there have been a number of papers that investigate sequential learning of service rates in queueing systems. These papers consider how bounds on regret (as defined above) relate to the performance of a queueing system.
We will review these shortly, but first we discuss the main ideas by developing a simple example.

\subsection{Regret and service in a queue.} 

We suppose that jobs arrive to a single server queue. We let $a_t$, $t \in \mathbb Z_+$, be the inter-arrival time between $t$-th arrival and the $(t+1)$-th arrival.
When we serve a job, we suppose that there are different modes of service that we can select. This represents the manner in which we are going to serve the job, with some service modes being more suitable than others. We let $\mathcal M$ be the set of service modes. 
A policy $\pi=(\pi_t : t \in \mathbb Z_+)$ choses a service mode for each job, i.e. for the $t$-th job at the queue we select service mode $\pi_t \in \mathcal M$.
We let $\tau^m_t$ be the time spend serving the $t$-th job under service mode $m$. We assume that
for each job, one service mode (which is not known to us) will lead to a fast service time, $\tau^\star = \min_{m \in \mathcal M} \tau^m_t$, and the other service modes will lead to a slower service time. % Here $\tau^\star < \tau^{0}$. 
We let $\tau_{t}^{\pi}$ be the time taken to serve the $t$-th job under policy $\pi$.
Shortly, in Section \ref{Percep}, we will discuss a specific model that relates the service modes to the speed of service. See Figure \ref{RegQ_Fig} for a simple example of model just described. 

Following our earlier definition, \eqref{Regret}, the regret of policy $\pi$ is 
\begin{align*}
  \mathcal R\! g(T) = \sum_{t=1}^T (\tau_{t}^{\pi} - \tau^\star )\, .
\end{align*}
I.e., the regret is the difference between the total service we receive under policy $\pi$ and the service that we receive under optimal service.
Given the inter-arrival times, $a_t$, and service times, $\tau^\pi_t$,
the waiting time of the $t$-th customer satisfies the famous Lindley recursion:
\begin{align*}
  W_{t}^\pi = \max \{ 0 , W^\pi_{t-1} + \tau_{t-1}^\pi - a_{t-1} \}\, .
\end{align*}
The waiting time under the optimal policy is then
\begin{align*}
    W_{t}^\star = \max \{ 0 , W^\star_{t-1} + \tau^\star - a_{t-1} \}\, .
\end{align*}
In both cases we assume that the queues are initially empty: $W_0^\pi = W_0^\star =0$. 

\begin{figure}[hbt]
\centering
  \includegraphics[width=0.35\textwidth]{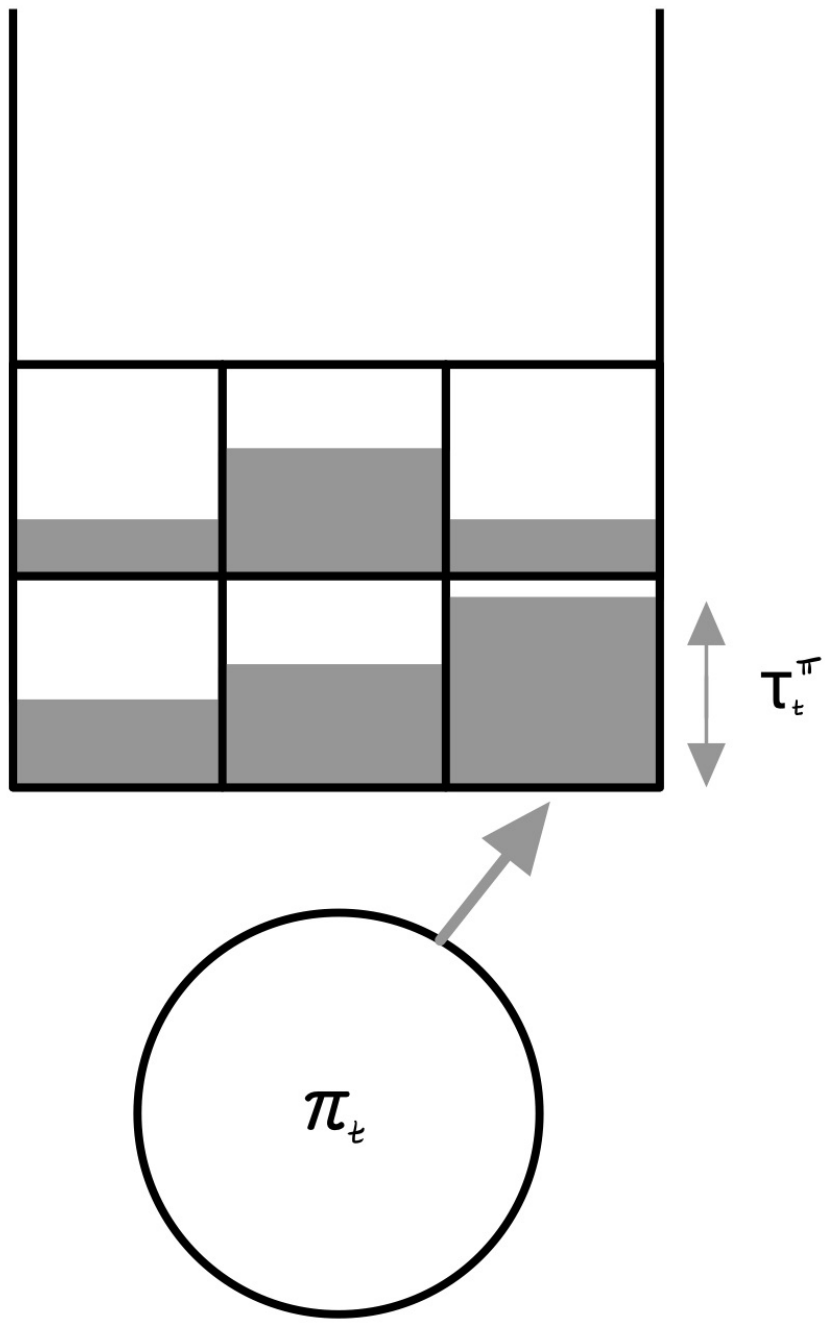}
  \caption{Regret at a queue. Here is a queue with two jobs and three modes of service. The current policy will serve the first job in under the 3rd service mode. This decision has a long service time. We can see that, in hindsight, the 1st service mode would have been best.\label{RegQ_Fig}}
\end{figure}

The following simple lemma provides a way to compare the waiting time of jobs in terms of the regret of the policy $\pi$

\begin{lemma}\label{LindLem} If $t_0$ is the last time before $t$ where the queue for policy $\pi$ is empty, i.e. $W^{\pi}_{t_0} = 0$, then
\begin{align}\label{W:Regret}
  W^{\pi}_{t+1} -  W^{\star}_{t+1} \leq  \mathcal R\! g(t) - \mathcal R\! g(t_0-1) 
\end{align}
\end{lemma}

With this  result in hand, there are a host of ways that we can apply regret analysis to a queueing model. In informal terms, here are two ways that one can proceed.
First, notice that if $t_0 = o(t)$, i.e. queues do not empty often, then the term $\mathcal R\! g(t) $ will dominate the above expression, \eqref{W:Regret}. Thus we have a form of rate stability. Specifically
\begin{align}\label{o(t)}
  \text{if }\quad W^{\star}_{t} = o(t) \quad \text{then } \quad W^{\pi}_{t} = o(t)\,.
\end{align}
So if there is \emph{some} policy under which queues do not grow linearly, then the queues will not grow linearly under $\pi$, either. This is essentially the observation made by \cite{walton2014two}.
This is the form of the result that one can expect in an adversarial setting. Second, in a stochastic setting we can expect a queue to be positive recurrent, that is $t-t_0 = O(1)$. Thus the bound is of the form of the expected change in $\mathcal R\! g(t) - \mathcal R\! g(t_0-1)$. In a stochastic multi-arm bandit problem, the regret is often the order $\log t$, So $\log t - \log( t_0-1) \sim \frac{t-t_0}{t}$, and we can expect that
\begin{align}\label{1/t}
  \text{if }\quad {t-t_0} = O(1)\quad\text{and}\quad  \mathcal R\! g(t) \sim \log t \quad \text{then}\quad
W_{t+1}^\star = O\left( \frac{1}{t}\right) \,.
\end{align}
This is essentially the observation made by \cite{krishnasamy2021learning} for a specific algorithm, although the same result likely holds in more generality. 

As the observations \eqref{o(t)} and \eqref{o(t)} suggest, there are a host of algorithms that are likely to permit exact analysis; however, one classical learning algorithm that permits a particularly tractable analysis is the perceptron algorithm, which we discuss next.

\subsection{Analysis with the Perceptron Algorithm}
\label{Percep}
We consider the setting where there are two modes of service $\mathcal M = \{-1,+1\}$. With each job, $t$, we associate a context $x_t$ and a response $y_t$. We assume $x_t$ is a bounded vector in $\mathbb R^p$, i.e. $||x_t|| \leq D$ for some $D$, and we assume that $y_t$ belongs to $\mathcal M= \{-1,+1\}$.
The context summarizes information about the job and with this information we are allowed to make a service mode decision, $\pi_t \in \mathcal M$.  We assume that the sign of $y_t$ determines whether the job receives fast or slow service. Specifically, if the sign of $\pi_t$ and $y_t$ match then the job receives fast service and otherwise it receives a slower rate of service.  That is we assume that
\begin{align*}
  \tau_t^\pi = 
\begin{cases}
	\tau^\star, &\text{if } \pi_t = y_t\, , \\
\tau^0, &\text{if } \pi_t \neq y_t\, . \\
\end{cases}
\end{align*}
Here $\tau^0$ corresponds to the slow service time and so $\tau^0 > \tau^\star$.
This correctly classifies the context of a job and leads to faster service. 

We restrict ourselves to policies that apply weights to each context in-order to make a decision. Specifically, for weights $w\in \mathbb R^p$, the policy  $\pi_w$ with context $x$ makes the decision 
\begin{align*}
  \pi_w(x) =
\begin{cases}
	+1, &\text{if } \langle w, x\rangle \geq 0 \, ,\\
-1, & \text{if } \langle w, x \rangle < 0 \, . \\
\end{cases}
\end{align*}
For weights $w$,
the loss of the $t$-th job is
\begin{align*}
  \hat l_t(w) = \mathbb I [\, \pi_w(x_t) \neq y_t ] \, .
\end{align*}
We can upper-bound this by the hinge loss
\begin{align*}
   l_t (w) = \max \left\{ 1- y_t \langle w, x_t \rangle, 0 \right\} \, .
\end{align*}
It is not hard to check that $\hat l_t(w) \leq l_t(w)$. The resulting optimization 
\begin{align*}
  \min_{w\in \mathbb R^p} \quad \sum_{t=1}^T \max \left\{ 1- y_t \langle w, x_t \rangle, 0 \right\}
\end{align*}
is known as a support vector machine.\footnote{Often an additional regularization term is added to this objective though we do not consider that here.}
Note that
\begin{align*}
  \nabla_w  l_t = 
\begin{cases}
	- y_tx_t & y_t \langle w, x_t \rangle < 1 \, ,\\
0 & \text{otherwise.}
\end{cases}
\end{align*}
If we apply an online gradient descent algorithm, we get the following update algorithm for the weights $w_t$:
\begin{align*}
  w_{t+1} = w_t + \alpha y_t x_t \mathbb I [y_t \langle w_t, x_t \rangle < 1] 
\end{align*}
where here $\alpha>0$ is the learning rate of the algorithm and we take $w_0=0$. 
The above algorithm is known as the perceptron algorithm. 

Allowing the learning rate to depend on time, 
the following is a standard bound in online convex optimization (OCO):
\begin{proposition}[OCO bound]\label{OCO:Lem2}
	 For any function $\nabla l_t(w)$ and $w_t$, as above, then for all $w \in \mathbb R^p$
	 \[
	 \sum_{t=1}^{T} l_t(w_t) - \sum_{t=1}^{T}  l_t(w)
	 \leq 
\frac{|| w_1 -w ||}{\alpha_1}+
	 \sum_{t=1}^{T} \frac{||w_t - w ||^2}{2}
	 	\left(\frac{1}{\alpha_t}-\frac{1}{\alpha_{t+1}}\right)
	 	+
	 	\sum_{t=1}^{T} \frac{\alpha_t}{2} || \nabla l_t(w_t) ||^2 \, .
	 \]
\end{proposition}
The above result bounds the performance of the sequence of weights $w_t$ in comparison with any fixed choice $w$. Thus the above bound is a regret bound. Appropriate conditions on the boundedness of $w_t$ and the magnitude lead to sharp regret bounds. A more general overview is given by \cite{hazan2019introduction}.
 
We pursue this now in the context of the Perceptron Algorithm. 
We consider the Perceptron Algorithm with a fixed learning rate, $\alpha$, we note that since $w_0=0$. The sequence of weights $w_t$ are proportional over each choice of $\alpha$. 
A consequence of the previous bound and this observation is the following classical result:

\begin{theorem}[Perceptron Mistake Bound]
\label{PMB}
	If there exists $w^\star \in \mathbb R^p$ such that $y^{(t)}\langle x^{(t)},w^{\star} \rangle  \geq 1$ then the perceptron algorithm is such that 
\begin{align*}
  \sum_{t=1}^\infty \hat l_t(w_t)   \leq {D}^2 ||w^\star||^2\, .
\end{align*}
\end{theorem}

If we apply the perceptron algorithm to the service of rate of customers as considered in Lemma \ref{LindLem} then, since the number of mistakes made by the algorithm is finite, the difference in the regret of the algorithm is bounded by a constant. Thus applying  Lemma \ref{LindLem}, we arrive at the following result.

\begin{corollary}\label{QOpt}
For weights determined by the perceptron algorithm then the waiting time of customers at the queue is
\begin{align*}
  W^{\pi}_{t+1} -  W^{\star}_{t+1} \leq {D}^2 ||w^\star||^2\, .
\end{align*}
Moreover if $t'$ is the first time that  $W^{\pi}_{t'}=0$ after the last mistake of the perceptron algorithm, then for all $t\geq t'$
\begin{align*}
  W^{\pi}_{t} =  W^{\star}_{t}\, .
\end{align*}	
\end{corollary}

The analysis here is particularly clean as the number of mistakes made by the perceptron algorithm in this ``{well-separated}" setting is finite.\footnote{Here well-separates means that there exists $w^\star \in \mathbb R^p$ such that $y^{(t)}\langle x^{(t)},w^{\star} \rangle  \geq 1$.} In general the difference in aggregate loss (or number of mistakes) will continue to increase, and the change in losses with decrease rather than be zero after some point. In this case, slightly more careful bounds need to be applied to gain similar results, although the main idea remains the same.

\subsection{Literature}

 %is a classical queueing recursion given in \cite{lindley1952theory}. 
%
The application of Lindley's recursion to the regret analysis of a queueing system follows the discussion \cite{walton2014two}. Here a result similar to Lemma \ref{LindLem} is derived and the regret for a variety of expert and bandit algorithms is considered. The results of  \cite{walton2014two} are primarily concerned with the adversial arrival and service types. The regret of stochastic bandit problems are considered in \cite{NIPS2016_430c3626}. Here a specific algorithm called $Q$-UCB is considered and a detailed analysis of the transition from instability to stable behavior is considered. The discussion in display \eqref{1/t} is a heuristic derivation of this result in \cite{NIPS2016_430c3626} and suggests that the results hold more broadly than for the $Q$-UCB algorithm. 
The work is developed further by the authors to consider multiple queues and matching constraints \citep{krishnasamy2021learning}.
The adversarial setting is considered for a MaxWeight model by \cite{liang2018minimizing}. Further formulations for wireless networks \citep{stahlbuhk2019learning} and for network utility maximization are considered \citep{liang2018network}. A strategic adversarial model of queueing is considered by \cite{gaitonde2020stability}. Here the loss in capacity through adversarial coordination is the object of interest. The authors improve these results exploiting sub-modularity structure \cite{gaitonde2020virtues}.

The set of papers above are amongst the first papers to consider regret analysis within a queueing system. Nonetheless it should be noted that there is an extensive prior literature adressing related problems. For a discounted reward objective the Gittin's index is known to be a Bayesian optimal bandit algorithm, see \cite[Section 35]{lattimore2020bandit}. There are several works which develop a queueing approach to the multi-arm bandit problem. For instance see \cite{aalto2009gittins,jacko2010restless,scully2020simple} as well as the text \cite{gittins2011multi}.
Further adversarial models of queueing systems are considered in some detail initiated by the work \cite{borodin1996adversarial}.

The perceptron algorithm is a classical algorithm considered in machine learning \citep{novikoff1963convergence}. The mistake bound forms the basis for VC and Radamacher complexity bounds \citep{vapnik2013nature}. Online convex optimization and perceptron mistake bounds follow the proofs given in \cite{shalev2011online}. 

Further, it is worth noting that the analysis of regret only applied to algorithms applied in a online setting where parameters must be fitted statistically. I.e. a regret analysis is not relevant to models trained by simulation. Here only the quality of the final solution reached by the learning algorithm matters. If an algorithm is trained online then care must be taken to the potential changing dynamics and statistics of the environment. Such a view point is taken by \cite{besbes2015non} in the setting of stochastic approximation and by \cite{daniely2015strongly} in the non-stochastic case. %Another potential more robust

In practical applications, the effectiveness of an online learning algorithm applied to a queueing system may well depend on the application area. A learning algorithm will always take time to converge to the best decision on average. So the average decision must be a metric of interest and the algorithm must have time to converge. For instance, demand from TCP flows at an internet router will look to overload the router thus the adversarial framework considered in Section \ref{BAMW} would be more appropriate. However, more slowly varying applications such as ride-sharing and matching markets might be more appropriate to the framework described here. Here statistically well-defined customer classes can be learning and optimized over time. This approach is considered by \cite{hsu2021integrated} where the unknown payoff from service at a queueing system must be learned and maximized subject to servers capacity constraints.
 \cite{johari2021matching} considers a related line of work where server classes must be learned. % The integration of back-pressure and learning ideas is further given by \cite{kanoria2019blind}. 
The work of \cite{liu2020pond} consider an online dispatch problem as a sequential learning problem and combine an online learning analysis with virtual queueing ideas in order to manage system capacity constraints. 
A further area that combines inventory control and online learning is online stochastic bin packing. Here an unknown distribution must be sequentially learned and packed in to unit sized bins. We refer the reader to \cite{doi:10.1287/opre.2019.1914} and references therein. 

%\Neil{Discuss Strongly Adaptive regret}

\section{Role of Information}\label{information}

The algorithms introduced the previous sections generally have a ``learning'' flavor: the algorithm either explicitly gathers information about unknown quantities in the environment, such as in the case of bandits, or implicitly controls the system in a manner that mitigates the lack of knowledge of such information, such as Max-Weight being able to stabilize the queues despite not knowing the arrival rate. What exactly, then, is the information that is being learned, and how does such information shape the performance and algorithmic design? These questions will guide our discussion in the next two subsections.

At a high-level, by information we are referring to knowledge of uncertain values  in the system. As the exposition of \cite{lu2021reinforcement} discusses, it  is  instructive to further consider the following two categories of uncertainty: 
\begin{enumerate}
\item Epistemic information: this corresponds to static uncertain parameters and system model specifications. For instance, in a typical queueing network, uncertainty in the arrival rate or mean service time belongs to this categories. 
\item Aleaoric information: this corresponds to uncertain realizations and stochastic events that manifest over time, and in particular, is unknown even if we have a complete specification of the system model. For example, while we may know that the arrival counts in time step $t$ follows a Poisson distribution with mean 1 (epistemic information), the exact realization of the number of arrivals $A_t$ is still uncertain. Knowledge of the realized value of $A_t$ would correspond to aleaoric information. 
\end{enumerate}
For the purpose of our discussion, one can roughly think of epistemic information as about static quantities, whereas aleaoric information captures dyanmic realizations of randomness over time. 

\paragraph{Aleaoric information: past, present and future}
We will focus in this section on the impact of aleaoric information. We will have a broader discussion about the role of epistemic information in the next section in the context of reinforcement learning. Note that, the bandit-inspired algorithms and the Max-Weight algorithm both focus on addressing epistemic uncertainty. 

We will further orient our discussion along three sub-type of aleaoric information, determined by where the information is situated along the time axis, relative to the present time frame: 
	\begin{enumerate}
	\item Information about the future inputs, a.k.a.~predictive information. E.g., how many jobs will arrive and to what queues in the next 5 hours. 
	\item Information about the current system state, such as the queue lengths at various servers in a load-balancing model.  
	\item Information about the past, possibly stored as part of the algorithm's internal memory. 
	\end{enumerate}

%\tr{older notes}

%\begin{enumerate}
%\item Future information: focus on the setup in Spencer et al of using future information to improve admission control. Key result: the ``one-to-one'' duality between information and queueing delay. Therefore, the stochastic fluctuation in the arrivals and services either manifest in queueing delay, or must be observed before hand in the form of future information. Key take-away: the amount of information needed to achieve ``good'' performance is tied to the amount of stochastic fluctuation intrinsic in the original queueing system. Another take-away: using predictive information leads to potentially counter-intuitive policies (compared to online Markovian policies), and such policies could also challenge our notation of fairness (pre-emptive diversion when the hospital is empty). 
%\item Information about the current state. The ideas of current vs.~past information seem to be inherently intertwined. The reasons are as follows: in an Markovian system (MDP), if one has perfect state information, then memory is not needed due to the Markov property. However, as soon as information about the current state becomes imperfect, past observations become valuable as they now can help form a more accurate estimate of the current state, as well as that of future evolution. As such, the issue of present information and memory must be studied jointly. 
%\begin{enumerate}
%\item 
%\end{enumerate}
%\end{enumerate}

\paragraph{Measuring the value of information} One of the core questions when it comes to the role of information is its decision-theoretic value, expressed as the performance improvement that accessing such information may enable. For instance, we may seek to find a function $v$ such that, {given $x$ amount of future information, we may claim  to achieve $v(x)$ amount of performance improvement}. 

Two immediate questions come to mind. First, in various stochastic resource allocation problems, what how does the value $v(x)$ depend on the amount of information, $x$? In a centralized decision-making framework (e.g., barring competing selfish agents), more information cannot hurt the decision maker, so we would naturally expect $v(x)$ to be always non-negative and monotone non-decreasing. A more interesting question can be: how fast does $v(x)$ increase as the amount of information $x$ increases, if at all. As we will see in some of the results surveyed in this section, sometimes a seemingly small amount of information can lead to significant performance improvement. Conversely, performance improvement could quickly plateau after obtaining ``enough'' information, beyond which more information only brings small, marginal benefits. 

Once we understand how much benefit information can bring, an immediate follow up question would be how to achieve the optimal performance improvement $v(x)$ using efficient algorithms. A notion that is suitable for most engineering applications is to require that the resulting decision policy to be computationally tractable. A taller order might be warranted when the application involves humans, such as in a call center or health-care setting, we would ideally like the algorithm to  be simple and intuitive from the perspective of a human operator. 

\begin{figure}[h]
	\centering
	\includegraphics[scale=0.9]{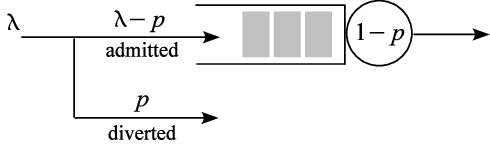}
	\caption{A queueing admission control model with future information \citep{spencer2014queueing}.}
	\label{fig:queue_future}
\end{figure}

\subsection{Improving Queueing with Future Information}

Our first example concerns  admission control in a queueing system using future information \citep{spencer2014queueing}. We will look at an operator who may access to a third-party oracle (via predictive algorithms, machine learning, or simply a side-channel) that offers a look-ahead window containing future arrival and departure patterns. We will see that in this example, not only does queueing delay improve with future information, but the function $v(x)$ experiences a sharp increase early on, before petering out, suggesting that the system can benefit a lot from having access to a moderate amount of information. 

The model is depicted in Figure \ref{fig:queue_future}. An single stream of jobs arrive to a queue following a Poisson process with rate $\lambda  \in (0,1)$. Service variability is modeled by an independent Poisson process of rate $1-p$: suppose that there is a jump in this process at time $t$, then the queue is reduced by $1$ if it is not empty, or stays at zero, otherwise. It is easy to show that in this system  the queue length process, $Q(t)$, would evolve according to the same law as the number of jobs in system (including in service) in a conventional $M/M/1$ queue, with arrival rate $\lambda$, service rate $1$, and mean job size $1/(1-p)$.

Now, notice that if the arrival rate $\lambda$ exceeds the service rate $1-p$  the queue length will eventually grow to infinity. We would thus assume that the policy designer is to make some admission decisions on a dynamic basis: for each arriving job, the algorithm is to decide whether the job is to be admitted to the queue, or \emph{diverted}, in which case the job leaves the system directly. Because diversions are costly (the diverted job is either not receiving the needed service, or needs to be served elsewhere), we impose a constraint that the average rate of diversion be at most $p$. This model of admission control has a long history and has been studied in various applications, ranging from Internet congestion control to ambulance divergence for Emergency Departments; see references in \citep{stidham1985optimal, spencer2014queueing, xu2016using}. 

A key question is how to control admissions in such a way that minimize the waiting time experienced by the admitted jobs, while obeying the diversion rate constraint.  Everyday intuition would suggest that jobs should be turned away when the queue is already very long, and conversely, we should probably admit the job if the queue is nearly empty. In other words, we could implement a threshold-policy where a job is admitted if any only if the queue length at the time of arrival is less than a certain threshold, and the threshold would be chosen to be smallest one that results in a feasible admission rate. Showing that this is indeed optimal is non-trivial, and was elegantly carried out in \cite{stidham1985optimal}. In particular, under the optimal policy (which uses a threshold rule), the average steady-state queue length scales as: 
\begin{equation}
\EE{Q}\sim \Theta\p{\log {\frac{1}{1-\lambda}}}, \quad \mbox{as $\lambda \to (1-p)$}.
\label{eq:fut_Q_online}
\end{equation}
Here, the limit $\lambda \to 1-p$ corresponds to the conventional heavy-traffic regime, where the arrival rate of the admitted traffic, $\lambda-p$, approaches the service capacity $1-p$. Note that as $\lambda \to 1-p$, the queue length explodes to infinity, as one would expect from standard heavy-traffic theory of queues, albeit at a more modest, logarithmic rate compared to the typical $1/(1-\lambda)$ heavy-traffic delay scaling.  

Let us now introduce future information in this model. Suppose now that the decision maker is equipped some advanced knowledge of the future arrivals and potential services in the form of a lookahead window: at any time $t$, we will assume that the algorithm is able to access the \emph{realizations} of both the arrival and service Poisson processes in the time window $[t, t+L)$. That is, the decision maker not only has at their disposal how long the queue is now, but how arrival and service will materialize in the next $L$ units of time.  We would expect that the availability of such future (aleaoric) information should allow the algorithm to plan more proactively and improve performance. But how large of an improvement can we get, and how much future information do we need in order to make such improvement possible? 

It turns out that future information does make a big difference. Specifically, it is shown in \cite{spencer2014queueing} that there  exists some universal constant $a>0$, such that if that the size of the lookahead window $L \geq a \log\frac{1}{1-\lambda}$, then the queue length under an optimal policy satisfies: 
\begin{equation}
\EE{Q} \to \frac{1-p}{p}, \quad \mbox{as $\lambda \to (1-p)$}.
\label{eq:fut_info_upper}
\end{equation}
That is, equipped with a sufficient amount of future information, the heavy-traffic queue length does not diverge, but instead converges to a finite constant, a sharp departure from the scaling in  \eqref{eq:fut_Q_online}. 

We have thus seen that having access to some future information could bring potentially massive performance improvement. However, \eqref{eq:fut_info_upper} is only half of the picture, as we might want to know whether  the same performance improvement can be attained with much less information. The answer turns out to be negative, as demonstrated by the following lowerbound by \cite{xu2015necessity}: there exists another positive constant, $b$, $b<a$,  such that if the lookahead window size $L\leq b \log\frac{1}{1-\lambda}$, then the queue length under \emph{any} algorithm must satisfy: 
\begin{equation}
\EE{Q}\sim \Omega \p{\log {\frac{1}{1-\lambda}}}, \quad \mbox{as $\lambda \to (1-p)$}.
\label{eq:fut_Q_future_lower}
\end{equation}

Comparing the above with \eqref{eq:fut_Q_online} and \eqref{eq:fut_info_upper}, we see that, quite surprisingly, the performance deteriorates to the same level as an online algorithm, as soon as the lookahead window  size $L$ drops  by even a  constant factor, $a/b$. In other words, $\log\frac{1}{1-\lambda}$ appears to capture a critical threshold for performance-relevant future information, where queueing delay is highly sensitive with respect to whether $L$ is greater or less than this level. 

The above results also reveal an interesting duality between information and stead-state queue length. While one may expect future information to help reduce the steady-state queue length, it may be surprising that they are related to each other so strongly as to be almost interchangeable. Examining \eqref{eq:fut_Q_online} through \eqref{eq:fut_Q_future_lower} shows that the system can essentially operate in one of two regimes, as the system load $ \lambda$ approaches $ 1 $: 
\begin{enumerate}
	\item \emph{Information-rich}: if the lookahead window size grows faster than $ \Theta \p{\log \frac{1}{1- \lambda} } $, then the steady-state queue length would be bounded,
	\item  \emph{Information-scarce}:  if the  lookahead window grows slower than $\Theta \p{\log \frac{1}{1- \lambda} } $, then the steady-state queue length would grow at rate $ \Theta \p{\log \frac{1}{1- \lambda} }$. 
\end{enumerate}
That is, the operate ``pays'' a cost of $\Theta \p{\log \frac{1}{1- \lambda} }$ either in information or queue length.

\paragraph{Insight on Policy Design} We may also be interested in knowing how policy design should adjust to the availability of future information and forecasts. Do we need to drastically revamp our usual intuition, or are we only looking at minor tweaks to existing policies? 

The reality is likely the former. One of the most striking findings of \cite{spencer2014queueing} is that not only is the optimal control with future information different from the one without, but the difference is so large that it would almost appear counter-intuitive. As discussed earlier, the  intuition for an online version of the admission control problem suggests that a good admission policy should follow some threshold rule and admit arrivals to the queue only when the queue is short. In contrast, \cite{spencer2014queueing} show that when future information is available, the optimal admission policy does almost the opposite: it is more likely reject jobs when the queue is short, and admit when the queue is long. In the extreme case where the lookahead window size is infinite, the optimal policy would \emph{only} reject jobs when the queue is empty. Why does future information lead to such drastic changes in optimal admission rules? 

In some sense, this departure is but an instance of a broader shift from being \emph{reactive} to \emph{proactive} as future information becomes part of our decision-making apparatus.  In the online setting, since the decision maker has no knowledge of the future realizations of arrivals and services, the build-up in the queue serves as the only warning of congestion. Consequently, the decision maker can react to bursts in arrivals only \emph{after} they have already occurred. When future information is available, however, the emphasis of the decision maker naturally shifts from what has happened to what will happen, and as such, the current ``state'' of the system (e.g., queue length) will carry less significance as the amount of future information grows. In the admission control problem, if the decision maker is to foresee future arrivals, then rejecting jobs that occur at the \emph{beginning} of a burst of arrivals will prove far more profitable than rejecting those jobs that arrive later in the same bursty period. This is because rejecting early arrivals reduces waiting for all subsequent arrivals, whereas rejecting the later arrivals does not have the same ``domino effect.'' It is for this reason that the optimal policy under future information tend to be more aggressive in turning away jobs when the queue is small; not because the queue's being small in itself forebodes future congestion, but simply that early arrivals in a burst of jobs tend to, by definition, arrive when the queue is small. Clearly, with a lack of foresight, it would not have been possible for an online policy to have such discernment. 

While these specifics ways in which policy design changes as a result of more future information may or may not generalize to other stochastic systems, a broader point to be made here is that variations in information structure likely will require us to re-evaluate our usual intuition, and sometimes in fairly drastic ways.

 \subsection{Communication and Memory in Load Balancing} 
 
Our second example from \cite{gamarnik2018delay} shifts the focus of information from the future to that of the present and past. Here, the decision makes may not even have the complete information about the current state of the system. A typical application that motivates such consideration is that of a large data center, where it is often expensive, if not impossible, for a scheduling algorithm to have at its disposal the real-time state  information on all servers and queues. In this case, the lack of accurate state information orignates from two possible sources: 
\begin{enumerate}
\item Limited communication: the decision maker may not be able to communicate frequently with all the queues and servers. 
\item Limited memory: there is not enough memory to retain all past information.
\end{enumerate}

\begin{figure}
\centering
\includegraphics[scale=0.25]{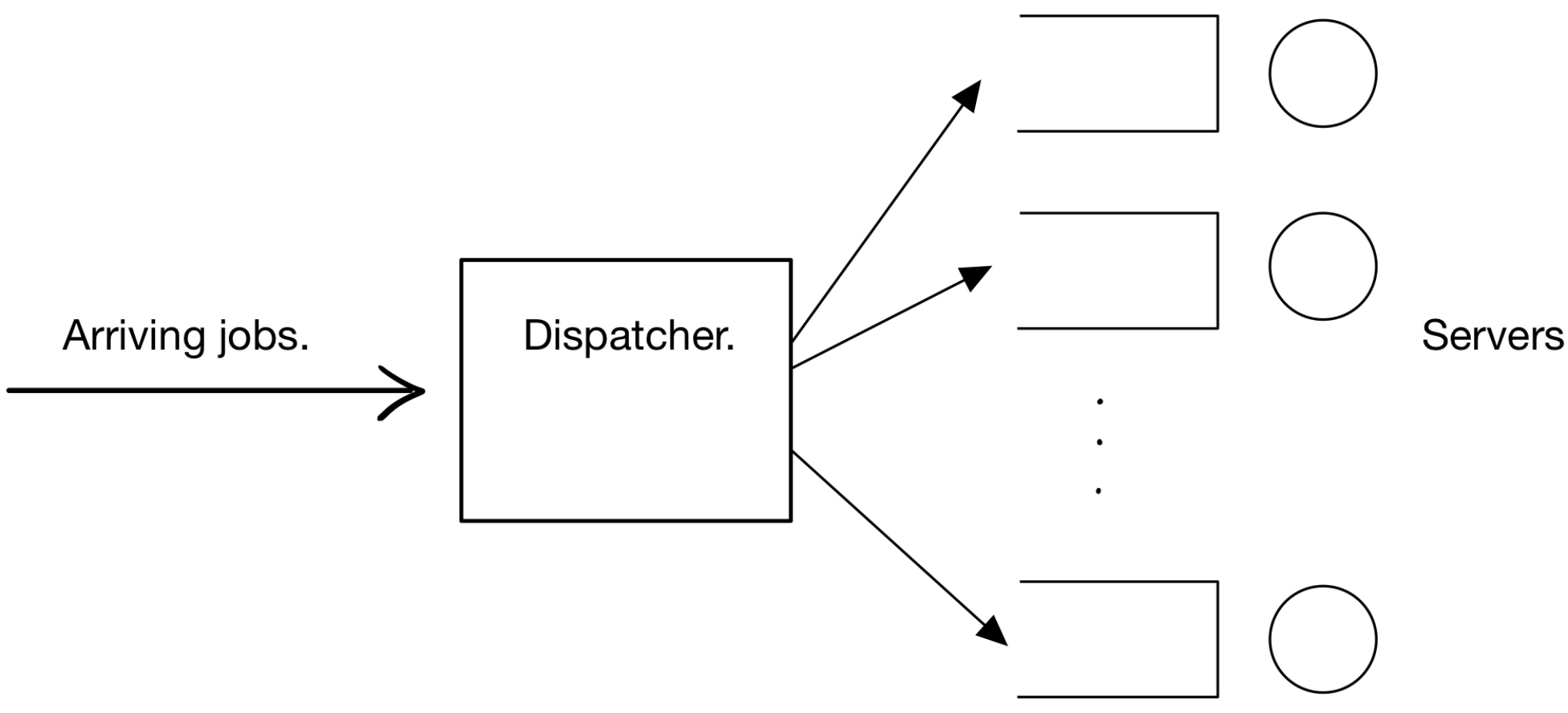}
\caption{Load balancing with a single dispatcher (Figure 1, \cite{gamarnik2018delay}).}
\label{fig:load_balance}
\end{figure}

To make matters more interesting, it turns out that the effects of communications and memory are often intertwined. With unlimited communication capability (and suitable Markovian assumptions), the true system state is readily accessible and there would be little need for memory. Conversely, when communications are limited, memory becomes crucial as it compensates for the lack of real-time information by piecing together last observations to form a better estimate of the system state.  Considering that communication and memory can both be prized resources in large-scale queueing networks,  how they jointly impact performance is an important question.
 
 The tension between communication, memory and system performance was  elegantly captured by a model proposed by  \cite{gamarnik2018delay} in the context of load balancing. The system is that of a standard load balancing setup, where a single dispatcher routes jobs, arriving at rate $\lambda n$, to a collection of $n$ parallel servers each operating at unit speed. The dispatcher however does not have access to the queue lengths at all the servers by default, but must rely on messages sent from the servers to gather such information.   A main innovation of this model is its systematic accounting of the communication and memory overhead. 
 
 \begin{enumerate}
 \item Communication model: when a server idles, it sends messages to the dispatcher according to a Poisson process of rate $r$. Each message contains the identity of the server encoded in a $ \log_2(n) $-bit binary string. The server sends nothing if it is currently busy. 
 
 \item Memory model: the dispatcher is equipped with a finite-sized memory bank of $c \log_2(n)$ bits.  Given that it takes $\log n$ bits to specify the index of a single server, this is effectively the same as the dispatcher being able to store the identities of $c$ (idle) servers. Whenever the dispatcher receives a new message, it adds the identity of that server to its memory, and discards it if the memory is already full. 
 \end{enumerate}

Under the above memory and communication model, \cite{gamarnik2018delay} study the delay performance of the following natural routing policy. Upon the arrival of a new job, the dispatcher randomly chooses a server ID from its memory bank, if it is non-empty, and sends the new job to the said server while erasing the ID from the memory. If the memory bank is empty, then the dispatcher simply routes the job to a randomly chosen server. The main result of \cite{gamarnik2018delay} characterizes the resulting delay as a function of the level of memory and communication available. In particular, they consider the regime as the number of servers $ n $ tends to infinity, and distinguish whether the average system-wide queueing delay vanishes in the limit. 

Their main results are summarized in Figure \ref{fig:delay_mem_comm}.  There are several interesting takeaways. First, as one would expect, delay improves as the amount of memory or message rate increase. For instance, in the ``High message regime,'' the messaging rate of each server tends to infinity as $ n $ increases. This high rate of message allows the system to achieve vanishing delay even when the dispatcher has only a limited memory size, capable of storing only a bounded number of server IDs ($ c >0 $ but bounded). In contrast, for the same memory size, if the messaging rate per server remains bounded, then delay becomes strictly positive in the large-server limit. The same complementarity between memory and communication is partially mirrored in the ``High memory regime,'' where the servers send messages at a constant rate of at least $ \lambda $, but the dispatcher can remember an unbounded number of server IDs. Here, the lack of communication is made up for by an abundance of memory, and vanishing delay is again achievable. 

A second, and perhaps more striking, insight is that the roles of memory and communications in this model are not entirely symmetrical. In the ``High memory regime,'' the messaging rate per server is assumed to be bounded. In this case, having more memory alone is not sufficient for achieving vanishing delay. In particular, assuming that the number of server IDs the memory can store grows to infinity as $ n $ increases ($ c \to\infty $), then vanishing delay is achieved if and only if the messaging rate per server is at least $ \lambda $. Recall that $\lambda$ is exactly the average arrival rate per server. So another interpretation of this result is that the servers need to communicate its idling status  at the rate of \emph{one message per job} in order for more memory to have a meaningful impact on delay performance. 

It is worth-noting that the analysis carried out in \cite{gamarnik2018delay} applies to the specific dispatching policy. While it is conceivable that some other policy could utilize the same memory and communication resources more efficiently, a follow up paper \citep{gamarnik2020lower} shows that any potential improvement will be limited. They prove that under a certain natural symmetric assumptions, any policy operating in the ``Constrained regime'' (see Figure \ref{fig:delay_mem_comm}),  where the messaging rate per server stays bounded and the memory size is on the order of $ \log n $, necessarily incurs a strictly positive limit delay as $ n\to \infty $.

\begin{figure}[h]
\centering
\includegraphics[scale=0.3]{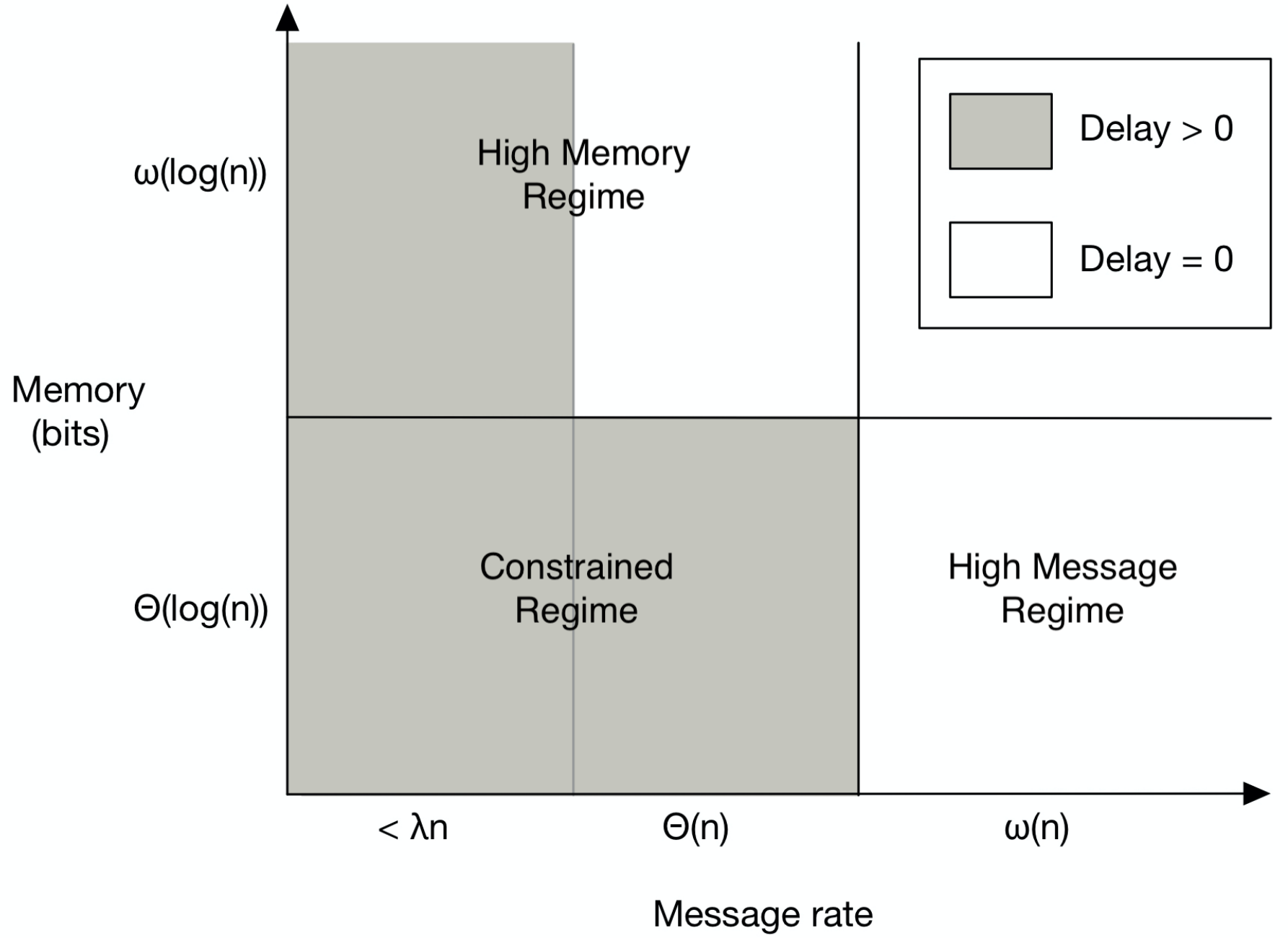}
\caption{Delay as a function of memory and communication (Figure 3, \cite{gamarnik2018delay}).}
\label{fig:delay_mem_comm}
\end{figure}

The interplay between memory, communication and system performance in a queueing network has also been investigated by \cite{xu2020information} in a different system model. They study dynamic resource allocation in a single-hop stochastic processing network model under imperfect communication and limited memory. Unlike \cite{gamarnik2018delay}, where communication is captured by messaging rates, \cite{xu2020information} considers a more general communication model using memoryless Shannon channels in information theory, which would allow for capturing both rate-limited and noisy communication. The focus of \cite{xu2020information} is centered around understanding how to achieve the maximum throughput region when the decision maker can access the state of a queueing system only through a noisy and rate-limited channel, when the algorithms are memory-limited.  The main results of \cite{xu2020information} consist of a family of Episodic MaxWeight policies that (approximately) achieve the maximum stability region under various combinations of memory and channel conditions. As another departure from the above load-balancing model, in which only the dispatcher possesses a memory, both the queues and  scheduler in \cite{xu2020information} have their own memory bank. A surprising conclusion from the results in \cite{xu2020information} is that it is not only \emph{how much} memory that matters, but also \emph{where} the memory is located:  granting more memory to the scheduler, rather than to the queues, leads to a significantly larger impact on the system's stability region.

\subsection{Literature Review and Discussions}

There has been a growing body of literature that examines the impact of predictive information in queueing networks. \cite{xu2016using} studies admission control in Emergency Departments where the operator may have access to limited and noisy future information. They propose a family of admission control policies that provably outperforms the optimal online policy at all arrival rates, in contrast to the policy in \cite{spencer2014queueing} which is only optimal in the heavy-traffic limit as $ \lambda \to 1 $. Using simulations based on historical hospital admission data, \cite{xu2016using} also demonstrate that leveraging future information can lead to substantial improvements in queueing delay over online policies even if the predictions are noisy and the lookahead window size limited.   \cite{ata2020optimal} applies a version of the model with future information of \cite{spencer2014queueing} to the problem of admission control in a call center with call-back options. A key insight there is that the type of proactive admission policies that are shown to be optimal in the over-loaded setting of \cite{spencer2014queueing} ($ \lambda > 1-p $) remains effective in an under-loaded system ($ \lambda < 1-p $).

In some cases, having information about future arrivals goes hand-in-hand with being able to \emph{intervene} on a more proactive basis. \cite{huang2015backpressure} examine proactive scheduling policies that improve the standard BackPressure algorithms in a queueing network. In this model, future arrivals to a queue can be not only predicted, but also served, in advance. In a similar spirit, but different domain,  \cite{hu2020optimal} uses a queueing model to study how to provide proactive care in a healthcare application, where the care provider strives to achieve better treatment outcomes by proactively intervening on a patient who is at the risk of deteriorating in the near future. 

There is a vast body of literature on dispatching policies  for load balancing that aim to reduce the communication overhead while maintaining desirable delay performance, including the celebrated power-of-$ d $-choices algorithm, where the dispatcher queries $ d $ randomly sampled servers and sends the arriving job to the least loaded among them. \cite{mitzenmacher2002load}, and more recently \cite{anselmi2020power, hellemans2020performance}, study the impact of dispatcher's memory on the performance of this class of algorithms. To the best of our knowledge, \cite{gamarnik2018delay} is the first to formally model the joint impact of communication \emph{and} memory. 

%While the paper largely focuses on a specific dispatching policy, a latter work by \cite{gamarnik2020lower} shows that the relative large delay from the policy operating with limited communication and memory in  \cite{gamarnik2018delay} is likely universal, not policy-specific. 

%\newpage

\section{Reinforcement Learning in Queues}\label{sec:RL}

%We discussed in the previous section how increased access to aleatoric state information, future or present, can improve performance. We now turn our attention to the challenges, as well as novel approaches,  in applying reinforcement learning to queueing systems, where learning  epistemic information plays a central role. The discussion here will integrate elements of the earlier sections, such as how well-known scheduling algorithms namely MaxWeight and BackPressure can be related to online learning, and how online learning and bandit algorithms can be applied to assist with learning various vital system parameters.

We now discuss the challenges in applying Reinforcement Learning theory to Queueing systems.
This discussion acts as a literature review of recent progress in the field. We highlight novel approaches in applying reinforcement learning in the context of queueing systems. 

We will assume some knowledge of the formalism of Markov decision processes and reinforcement learning, such as $Q$-learning, policy gradients, function approximation, as well as the important role of simulation in training models. 
To very briefly summarize, a Markov decision process (MDP) is a Markov processes where the evolution of the process is also influenced by actions which are chosen at each time step by a policy. The purpose of the policy is to maximize the expected cumulative rewards received by the process. The optimal policy and its value can be calculated using the Bellman equation. In practice, the evolution of the MDP and its rewards is often unknown and instead must be estimated from data. 
The suite of algorithms for this joint estimation and optimization is referred to as reinforcement learning (RL). 
Often the MDP objective has an unreasonably large complexity, thus different forms of function approximation must instead approximate the Bellman equation or directly perform an approximate policy optimization on the MDP objective.
Since the early '90, and particularly over the last 5 years, neural networks have been a commonly used framework for function approximation.
Successful policies can require large amounts of data to be learned. 
Thus policy training is often performed via simulation.
We refer the uninitiated reader to text such as \cite{sutton2018reinforcement}, \cite{bertsekas2019reinforcement} and \cite{szepesvari2010algorithms} for introductory accounts. 

Below we emphasize the role of queueing models as an important application area for early reinforcement learning models. We then discuss research challenges in the theory of reinforcment learning when applied to queueing systems. We review recent literature that addresses a number of these challenges and finally emphasize their importance in emerging application areas. 

\subsection{Queueing in Early RL Literature.}
%??We discuss the challenges that have been highlighted in the context of applying??.
%We discuss these points in the text below. 
The application of dynamic programming and reinforcement learning to queueing systems has existed almost since inception. A good example is the work of \cite{crites1996improving}. The paper considers allocation of a set of elevators in a building in order to minimize waiting time.  The work is amongst the earliest examples of Q-learning with neural networks.
The paper highlights a number of specific challenges that can occur within a queueing system:
time is continuous and so the times at which decisions are made is a part of the modeling process (in order to regulate the branching factor); there is often incomplete state information (for instance about the number of people waiting); there is decentralized coordination required between different agents; and the stochastic and potentially non-stationary arrival of passengers must be modelled. Many of these issues remain challenges for applying reinforcement learning in queueing systems. 

Works of \cite{zhang1995reinforcement} consider the application in job shop scheduling problems. Here complex planning constraints are approximated. \cite{singh1997reinforcement} gives an early example of reinforcement learning applied to dynamic channel allocation. Here a channel assignment problem in a cellular telecommunication system is modeled as a dynamic program and a reinforcement learning solution with function approximation is found to outperform the best heuristics.
The application of neural network reinforcement learning to inventory management is considered by \cite{van1997neuro}.
 In addition to neural network approximations, the cross entropy method when first applied to reinforcement learning was applied to inventory control as an example application by \cite{mannor2003cross}. 
From these works it is clear that after inception new reinforcement learning methods can readily be applied to a queueing system.

The success of these methods does depend greatly on the predictability of the underlying queueing process. There is a risk of over-optimizing a system for incorrect parameters. As a related anecdotal example, we consider the routing of telephone calls. Here AT\&T would perform periodic learning and optimization in order to determine dynamic routing of telephone traffic, see \cite{391435}. However, it was found that Dynamic Alternative Routing (DAR) \cite{gibbens1988dynamic}, where routing decisions are made based on local blocking decisions, was found to be much more successful. Here, centralized learning and optimization of (epistemic) parameters performed worse than decentralized state dependent (aleatoric) optimization. This is because the local state provided a better representation of the decisions to be made rather than the estimation of the state evolution. The role of current information in decision making, as discussed in the previous section, remains important here. 

As we will discuss next, it continues to be the case that queueing systems provide a rich set of models where simulation and reinforcement learning methods are applicable. Nonetheless it is also reasonable to say that even in simple settings mathematical guarantees, correctness and convergence of reinforcement learning algorithms is not well understood.\footnote{Beyond canonical assumptions such as finite state space with each state having a lower bound on the proportion of time visited.} %We discuss progress on this point after the next subsection. 

\subsection{Challenges in applying Reinforcement Learning Theory to Queueing Networks.}
\label{RLChall}
There is a growing interest in the application of reinforcement learning to queueing systems, 
but there are evidently challenges. We need algorithms that learn parameters that we know to be fixed but provide flexibility to model variability. To do this, as the previous section suggests, models require correct and extended notions of state information. For instance information of future arrivals can improve decision making for these transient parameters, or providing contextual information on service rates. However, decentralization, continuous time modeling, incomplete state information, and simulation model misspecification can all inhibit our ability to leverage the benefits of reinforcement learning in a queueing system. 

Nonetheless the potential practical power of reinforcement learning methods remains appealing. Many features of a queueing model cannot be well modeled in closed form and so simulation- and data-driven solutions have the ability to capture features that are not well represented by handcrafted models. (For example, it would be better to predict the speed up and slow down of vehicles at a junction rather than create a closed-form mathematical model.)
Before discussing recent literature, we highlight challenges that apply to the development of reinforcement learning for queueing systems:
\begin{itemize}
\item \textbf{Unbounded State Space.} Queueing systems often have unbounded state spaces, while theoretical results on reinforcement learning assume large but finite state spaces. 
Function approximation can help alleviate this; however, simulations for small queue sizes may lead to sub-optimal performance for a system in higher-load, which is often the scenario that the system designer cares most about. 
\item \textbf{Unbounded State Cost.} Even if function approximation may assist, queueing systems can have long regenerative cycles and thus large mixing times. Thus, if we wish to minimize queue size or delay, then the associated costs will be unboundedly large.
\item \textbf{Stability Guarantees.} We may wish to guarantee a level of first order performance. We may wish to optimize cumulative rewards subject to stability over a range of loads. There are relatively few guarantees of robustness to parameter uncertainty in a trained reinforcement learning model.
\item \textbf{Non-stationarity.} Parameters may change in time. Daily variations can be modeled through the use of contexts. However, in settings such as communications, fluctuations in load can be erratic, which again leads to the risk of learning one part of the state space at the cost of performance out of sample.
\item \textbf{Decentralized Coordination.} A queueing systems often consist of interconnected, geographically and logically diverse components. Frequently, centralized coordination is not possible and simulation of the whole system may not be possible either. The optimization of individual components should contribute to the overall optimality of the system as a whole.

%\item \textbf{Real world vs Simulation.} %\item \textbf{Online vs Offline.}

\item \textbf{Dependence on Simulation.} As with other application areas such as robotics, a simulation driven solution may used to train models but these simulated behaviors might not found in reality. In a road traffic setting, a simulator might incorporate the physical changes of a vehicle but may not account for events such as change lanes or drivers accelerating through an amber light. 

\item \textbf{Continuous Time Modeling.} Although it is not an unassailable challenge, it is worth noting that many queueing applications are continuous time and thus appropriate decision epochs must be defined. There is a trade off between reducing branching factors while allowing for a suitably rich set of policies to be implemented over time. 
\item \textbf{Incomplete State Information.} In many settings, the state of a queueing system may not be fully observable or may only be locally observable. In internet congestion in one part of a system may only be sensed locally. In a data center a dispatcher may only be able to gain information on a subset of its servers. We may not know how may people are waiting at a traffic light when a button is pressed. 
\end{itemize}

The above list is by no means exhaustive but covers many issues that RL theory does not currently address for queueing systems.

%\subsection{Application Areas}
%
%\NSW{Application areas
%\begin{itemize}
%	\item \cite{dai2019inpatient} Patients Dai.
%\item \cite{feng2020scalable} Ride hailing
%\end{itemize} 
%}
%
%
%
%BELL LABS CONTROL MODELS AND QUEUEING NETWORKS
%
%\NSW{
%\begin{itemize}
% \item Discuss learning underlying arrivals
%\item Discuss how extending information may help alievate the above point
%\item Discuss how contextual bandits deal with time dependent change by making this a learnable feature (this is generally applicable in hospital modeling)
%\item Discuss infinite state space (theory issue)
%\end{itemize}
%}

\subsection{Recent literature.}

Here we review recent  theoretical works that look to bridge the gap between theory and practice for reinforcement learning applied to queueing systems.

The paper of \cite{shah2020stable} looks to address the issue of infinite state spaces found in queueing systems. The method considers a setting where their is know to be a Lyapunov function providing stability for some (optimal) policy. %Thus when learned the reinforcement learning problem effectively exists in a bounded set of states and then for this reason techniques applied for finite state space analysis can be applied. 
The approach pays particular emphasis on evaluating actions for, thus far, unvisited states. Here, the online algorithm does have access to an oracle (or simulator) that can predict future performance from the current state.
The paper of \cite{liu2019reinforcement} take a different approach by truncating the state-space on which reinforcement learning can be applied. Here it is assumed that there is a known but perhaps sub-optimal policy (such as Max-Weight) that can stabilize the queueing network. 
A standard RL algorithm is used for learning when the state space is sufficiently small; however, 
once the Lyapunov function associated with this policy increases beyond some threshold the sub-optimal stabilizing policy is applied. %The approach is natural.

Although not explicitly analyzing queueing networks, the work of \cite{qu2020scalable} does much to analyze the effects of decentralized learning and control in a network setting. Here is it proven that if separate nodes of a network implement actions and if the impact of actions at neighboring nodes is delayed, then applying reinforcement learning locally yields a policy that is within a constant factor of the globally optimal policy. The factor in the degradation can be given by either the discount factor of the MDP or by the mixing time of the Markov decision processes. 

The previous papers a focus on tabular reinforcement learning methods. However, in practice, function approximation is more commonly applied to yield a solution. 
The paper of \cite{dai2020queueing} analyses on the application of deep reinforcement learning to queueing systems. The paper considers the implications of proximal algorithms such as Proximal Policy Optimization and Trust Region Policy Optimization. Here algorithms estimate and optimize in a similar manner to the Policy Iteration algorithm.  \cite{dai2020queueing} focus is on the application to queueing problems with an average cost objective under a moderate to heavy load. The authors present a number of Lyapunov function arguments that modify the policy optimization step of their algorithm. Further they introduce new variance reductions techniques for their value function estimates. These improvements result in pronounced performance benefits for their algorithms in a queueing setting. %Further, a number of variance reductions techniques are introduced. 

In addition to proximal algorithms, policy gradient approaches have also been considered in the context of queueing examples by \cite{zaki2021improper} and \cite{zaki2021better}. Here the task of the MDP formulation is to discover the best mixture amongst a given set of controllers. Under deterministic gradient updates regret bounds are derived. 

An under area of investigation is if know formulas can be used to improve the optimization of a queueing system. Albeit not a strictly reinforcement learning setting, the paper of \cite{dieker2017optimal} should how closed form queueing expressions can be incorporated into stochastic gradient descent procedures in order to expedite the optimization of a queueing network model.

A similar approach exploiting the structure of the queueing system is to simply start with an MDP queueing system with a known solution and then to use this structure along with parameter learning in order to establish theoretical bounds. This approach is taken in the case of inventory control see \cite{agrawal2019learning}. Here the base-stock policy is known to be optimal and the value function and cost functions are know to be convex. Through this precise regret bounds are established for an online policy in comparison with the best base-stock solution. 

%\cite{shah2020stable} Stable Reinforcement Learning with Unbounded State Space.
% \cite{dai2020queueing} Dai Deep RL.

\subsection{Application Areas}

We review a range of queueing application areas where reinforcement learning approaches have been investigated. As discussed applications include healthcare, manufacturing, transportation and communications. The studies below for the most part focus on training reinforcement learning algorithms through simulation. 

\cite{dai2019inpatient} consider the application of reinforcement learning for the optimization of hospital patient inflows. Here the overflow decision problem is cast as an average cost MDP. An approximate dynamic program is formulated and verified. 
Emergency response and vehicle assignment approaches to have been applied as MDPs by \cite{van2018real}. 
Reinforcement learning formulations of this are considered by \cite{lopez2018distributed}.
\cite{feng2020scalable} develop the work of \cite{dai2020queueing} and apply this to ridesharing systems. See \cite{qin2019deep} for a review and discussion on the application of reinforcement learning to Didi's ride-sharing fleet. 

Traffic signal control is a further transport applications area. The SurTrac system in the US applies estimation along with a forward dynamic programming solution, see \cite{smith2013smart}. Companies such as Vivacity explicitly apply deep reinforcement learning in their traffic control system.\footnote{https://vivacitylabs.com/technology/junction-control/} See \cite{cabrejas2021reinforcement} for a recent study that demonstrates the ability of deep RL to outperform current commercial traffic signal controllers. 

Studies of distributed reinforcement learning for manufacturing processes are considered in \cite{qu2018dynamic}. Here the stochastic processing network model of \cite{dai2020processing} is considered within a reinforcement learning framework. Further studies on machine repair are given by server speed optimization are given by \cite{prabuchandran2016reinforcement}.  Semiconductor manufacturing is investigated by \cite{park2019reinforcement}.

Applications in communications systems and data centres was initiated by \cite{tesauro2005online}. See also \cite{tesauro2005utility} and \cite{tesauro2006hybrid} for further data centre applications of RL. The area has received further attention with Deepmind's application in Google Data centers.\footnote{\url{https://deepmind.com/blog/article/deepmind-ai-reduces-google-data-centre-cooling-bill-40}} A more recent study focusing more on congestion control protocols is given by \cite{tessler2021reinforcement}.

%The above references are intended to be indicative of applications of current interest. Regardless of theoretical challenges there are also challenges in more applied studies. 

The training reinforcement learning algorithms in large state spaces often requires simulated data. However, there are few standardized simulation environments for queueing systems. For instance, the generic simulators like Open AI's gym environment are required; the ORSuite simulator is a recent simulation environment that looks to bridge this gap, see \cite{ORSuite}.  Within specific fields there are respected simulation environments like ns3 simulator\footnote{\url{https://www.nsnam.org/}} for networking, SIMGrid for high performance computing\footnote{\url{https://simgrid.org}} or Vissim\footnote{\url{https://www.ptvgroup.com/en/solutions/products/ptv-vissim/}} and SUMO\footnote{\url{https://sumo.dlr.de/docs/index.html}} for traffic control. Also see \cite{pan2021high} for a recent simulator for hospitals. A hope is that established simulators and problem sets emerge, so that the performance of different polcies can be compared. 

Due to their structure some queueing problems still remain challenging for reinforcement learning approaches. For instance the DAG scheduling problem in heterogeneous computing remains challenging due the large array of potential DAGs that a platform must support. Nonetheless, dynamic programming heuristics remain popular, see \cite{topcuoglu2002performance}. Further, in many practical applications (air traffic control or hospital assignment) there is human level control which governs part of the dynamics and the costs and rewards of the process. Thus a set of solutions are likely required to be readily available. These factors may limit that ability to consider completely model-free learning to be applied some application areas.

Nonetheless a more empirical approach to research on queues and queueing networks will likely stimulate new research. A hope is that queueing theoretical research can be development more closely with the increased acceptance and use of exploratory studies.

%\NSW{Practical Considerations}
%\begin{itemize}
%\item Discuss simulation not modeling reality
%\item overfitting issues in RL (potential need to hardwire modelling assumptions.
%\item Queueing network simulator Whitt, ns3 simulator, SUMO and VISSIM.
%\end{itemize}
%\input{5_Application_areas}
\section{Conclusions}

In this tutorial, we have seen how classical queueing algorithms such as MaxWeight and Back Pressure can viewed as an application of Blackwell's Approachability Theorem. Through this we see connections between queueing and online learning theory. This online learning approach does not require the explicit statistical estimation of parameters. However, to consider more general models, we see that the underlying parameters such as service rates must be estimated. We discuss this next and develop connections between statistical learning of sequences of data and the Lindley's recursion for the G/G/1 queue. Through this we develop a notion of regret for a single server queue. As an example we showed how the perceptron algorithm could be used for the classification of customer service at a queue and we bound the impact that this has on the queue length of the policy implemented and the (unknown) optimal policy. %It would be interesting to develop these notions for networks of queues.
We then discuss the role of state estimation when compared with parameter estimation. Here we discussed two examples. The first considers the importance of future information in making access control decisions. The second considers the importance of memory in load balancing systems. Finally, we review the application of reinforcement learning highlighting the importance of reinforcement learning to solve practical queueing challenges as well as highlighting current gaps between theoretical results in reinforcement learning and their applicability to modern queueing problems. 

Certainly we anticipate that over coming years approaches to information, learning and control will grow in their applicability to queueing networks, and we hope that this tutorial provides a comprehensive introduction and reference to modern methods in this emerging area. 
\appendix
\section{Appendix}

\subsection{Proof of Blackwell's Approachability Theorem}
We now restate and prove Blackwell's Approachability Theorem. 
\begin{theorem}[Blackwell's Approachability Theorem]\label{Blackwell:Blackwell} The following are
equivalent
\begin{enumerate}
\item ${\mathcal Z}$ is approachable.
\item For every $q$ there exists $p$ such that $R(p,q)\in {\mathcal Z}$.
\item Every half-space containing ${\mathcal Z}$ is approachable.%\footnote{Recall a half-space in ${\mathbb R}^k$ is a set ${\mathbb H}=\{a\in{\mathbb R}^k : n\cdot a \leq c\}$ for some $n\in{\mathbb R}^k$ and $c\in{\mathbb R}^k$.} 
\end{enumerate}
\end{theorem}
%\begin{proof}[Proof of Theorem \ref{Blackwell:Blackwell}.]
%\proof{Proof of Theorem \ref{Blackwell:Blackwell}.}
\Beginproof

First we need to show that the decision $d(t)$ given in \eqref{dRa} exists. To this end, we recall that the Minimax Theorem states that for a $p\times q$ matrix $M$ it holds that
\begin{align*}
\max_{a\in \mathcal A}  \min_{d \in\mathcal D} d^\top M a = \min_{d \in\mathcal D} \max_{a\in \mathcal A} d^\top M a  \, .
\end{align*}
A less compact way of writing the Minimax Theorem is that for every $v\in\mathbb R$ 
\begin{align}\label{MM2}
  \forall a \in \mathcal A\,\,\exists d\in\mathcal D\, \,\, \text{ such that } d^\top M a \leq v \iff \exists d \in \mathcal D\,\,\text{ such that }  \forall a \in \mathcal A,\,\, d^\top M a \leq v \, .
\end{align}
(Informally this states that if there exists a good choice of $d$ for each $a$ then there exists a good choice $d$ for all $a$.)

Now if a half-space, $\mathcal H = \{ r \in \mathbb R^n : \hat n \cdot r \leq v \}$, that contains $\mathcal Z$ is approachable, then it must be that when the adversary choses $a \in\mathcal A$ there exists a choice $d \in \mathcal D$ such that 
\begin{align}\label{MMR}
  \hat n \cdot R(d,a) \leq v \,.
\end{align}
If this were not true then $\mathcal H$ certainly would not be approachable. 

The statement \eqref{MMR} can be recast in terms of the Minimax Theorem. Specifically, we can define the matrix $M$ with components $M_{ij} := \hat n \cdot R_{ij}$. Thus by definition $\hat n \cdot R(d,a) = d^\top M a$.
Equation \eqref{MMR} now states that $ \forall a \in \mathcal A\,\,\exists d\in\mathcal D$ such that 
\begin{align*}
d^\top M a \leq v  \, .
\end{align*}
Thus by the Minimax Theorem, \eqref{MM2}, there exists $d\in \mathcal D$ such that 
\begin{align*}
  \hat n \cdot R(d,a) = d^\top M a \leq v, \qquad \forall a \in \mathcal A.
\end{align*}
So, we now know that for all hyperplanes $\mathcal H$ containing $\mathcal Z$, there exists a $d\in \mathcal D$ such that 
\begin{align}\label{nR}
  \hat n\cdot R(d,a) \in \mathcal H, \qquad \forall a \in \mathcal A.
\end{align}
This establishes the existence $d(t)$ as described in \eqref{dRa}. 

 Next we show that the sequence $d(t)$, $t\in\mathbb Z_+$ is such that $\bar Q(t)$ approaches $\mathcal Z$. 
\begin{align}
& ||\bar Q({t+1})- P({t+1})||^2 \notag \\
\leq & || \bar Q({t+1}) - P(t)||^2  \notag \\
= & || \bar Q({t+1}) -\bar Q({t}) ||^2 + || \bar Q(t) - P(t) ||^2 +2 (\bar Q({t+1})-\bar Q(t))\cdot (\bar Q(t)-P(t)) \, \notag \\
\leq & 
2\frac{R^2_{\max}}{(t+1)^2}
+ || \bar Q(t) - P(t) ||^2 +2 (\bar Q({t+1})-\bar Q(t))\cdot (\bar Q(t)-P(t)) \, .
\label{Blackwell:parallellogram1}
\end{align}
The first inequality holds since the projection $P(t+1)$ is by definition closer to $\bar Q(t+1)$. The equality is the parallelogram rule.
The final inequality follows since,
by the definition of $R_{\max}$ and since
\begin{equation*}
	\bar Q({t+1})=\frac{1}{t+1}R(d(t+1),a(t+1)) + \frac{t}{t+1}\bar Q(t)\, .
\end{equation*}
Further,
applying the above equality, to the final term in \eqref{Blackwell:parallellogram1} gives
\begin{align}
& (\bar Q({t+1})-\bar Q(t))\cdot (\bar Q(t)-P(t)) \notag \\
&= 
\frac{1}{t+1}(
{R(d({t+1}),a({t+1}))} - {\bar Q_{t}} )\cdot ( \bar Q(t)-P(\bar Q(t)))  \notag\\
&= 
\frac{1}{t+1}\Big[ (R(d({t+1}),a({t+1})) - P(t) )\cdot ( \bar Q(t)-P(t)) - ({\bar Q({t})}-P(t))\cdot
(\bar Q(t)-P(t))  \Big] \label{Blackwell:sub1}
\end{align}
Substituting \eqref{Blackwell:sub1} into \eqref{Blackwell:parallellogram1} and taking expectations, we have
\begin{align}\label{Bwellmid}
\mathbb E\big[ 
&  ||\bar Q({t+1})- P({t+1})||^2 | \mathcal F_{t}\big] \notag \\
\leq &
2\frac{R^2_{\max}}{(t+1)^2}
+ \left( 1- \frac{2}{t+1} \right) || \bar Q(t) - P(t) ||^2 
+
\frac{2}{t+1}
(R(\bar d({t+1}),\bar a({t+1}))-P(t) )\cdot (\bar Q(t)-P(t))
\notag
\\
\leq & 2 \frac{R^2_{\max}}{(t+1)^2} + \left( 1- \frac{2}{t+1} \right) || \bar Q(t) - P(\bar Q(t)) ||^2
\end{align}
where in the 1st inequality above we note that
\begin{align*}
  \mathbb E [ R(d(t+1) , a(t+1)) | \mathcal F_t ] = R(\bar d(t+1) , \bar a(t+1) ) 
\end{align*}
and in the 2nd inequality above we use the fact that 
\begin{align*}
(R(\bar d({t+1}),\bar a({t+1}))-P(t) )\cdot (\bar Q(t)-P(t)) \leq 0
\end{align*}
which holds by our assumption on $d(t)$, \eqref{dRa}.
%, and we use the fact that 
%\begin{align*}
%  ||\bar Q({t+1})-\bar Q({t})||^2
%=\frac{|| R(d(t+1),a(t+1)) -\bar Q(t)
%||^2}{(t+1)^2} \leq \frac{R^2_{\max}}{(t+1)^2} \, .
%\end{align*}

Multiplying both sides of \eqref{Bwellmid} by $(t+1)^2$ and rearranging gives
\begin{equation*}
(t+1)^2 \mathbb E [ ||\bar Q({t+1})- P({t+1})||]^2 - t^2 \mathbb E [||\bar Q({t})- P({t})||^2 ] \leq 2 R^2_{\max} -  \mathbb E [|| \bar Q(t) - P(t) ||^2]
\leq 2 R_{\max}^2.
\end{equation*}
Summing these interpolating terms gives 
\begin{equation*}
t^2 \mathbb E [ || \bar Q(t) - P(t) ||^2] \leq 2R^2_{\max} t  \, .
\end{equation*}
Thus, as required,
\begin{equation}\label{blackwell: conv}
d(\bar Q(t) , \mathcal Z) = \sqrt{ \mathbb E [|| \bar Q(t) - P(t) ||^2]} \leq R_{\max}\sqrt{\frac{2}{t}}\xrightarrow[t\rightarrow\infty]{} 0.
\end{equation}
%\end{proof}
\Endproof

\subsection{Hannan Gaddum Theorem}

\begin{theorem}[Hannan-Gaddum Theorem]
%Suppose $R(p,q)$ as defined in \eqref{blackwell:A} is such that $R(p,q)\in{\mathbb R}$. 
%As before, we consider
%an adversary where the sequence $\{q_t\}_{t=1}^\infty$ is either Arbitrary[D[D[D[D[D[D[D[a and
%predetrimned, or where $q_{t+1}$ is a function of the player's choices $p_1,...,p_t$.
There exists a playing strategy $\{d_t\}_{t=1}^\infty$ such that for any $\{a_t\}_{t=1}^\infty$
\begin{equation}\label{blackwell:regret}
\limsup_{T \rightarrow \infty}\; \mathbb E \left[ \frac{\mathcal R\!g(T)}{T} \right] \leq 0. 
\end{equation}
In other words, our performance in the game is asymptotically as good as the best fixed action.
\end{theorem}
The proof follows as a consequence of Blackwell's Approachability Theorem.
\Beginproof
We define the vector payoff
${R}(d,a)=(r(i,a)-r(d,a)\; :\; i=1,...,p)$ and convex region ${\mathcal Z}=\{R: R_i\leq 0, \forall i=1,...,p\}$.
For all $a$ there exists $d$ such that component-wise ${r}(d,a)\leq 0$, in particular we choose
$d$ to be the probability distribution with $d_{i^*}=1$ where $i^*=\argmax_{i=1,...,p} r(i,a)$. This verifies Condition 2 of Blackwell's
Approachability Theorem. Thus there exists a strategy $\{d(t)\}_{t=1}^\infty$ such that for all
$\{a(t)\}_{t=1}^\infty$
\begin{equation*}
D({\mathcal R \!g(T)}/{T} , \mathcal Z) 
\xrightarrow[ T\rightarrow\infty]{}
 0\, .
\end{equation*}
As a consequence, we can analyse the expected regret:
\begin{align}
  \mathbb E \left[ \frac{\mathcal R\! g(T)}{T} \right] 
&=
\mathbb E \left[ 
\max_{i=1,...,p}\sum_{t=1}^T r(i,a(t)) 
-
\frac{1}{T}   \sum_{t=1}^T r(d(t),a(t))  
\right]
\label{RGeq1}
\\
& \leq  
\mathbb E \left[ 
\sum_{i=1}^m \frac{1}{T} \sum_{t=1}^T \big| r(i,a(t) ) - r(d(t),a(t))  \big|_+
\right]
\label{RGeq2}
\\
& 
\leq 
\mathbb E \left[ 
\sum_{i=1}^m \frac{1}{T} \sum_{t=1}^T \big( r(i,a(t)) - r(d(t),a(t))  \big)^2_+
\right]^{\frac{1}{2}}
\notag
\\
&= D({\mathcal R \!g(T)}/{T} , \mathcal Z) 
\xrightarrow[ T\rightarrow\infty]{}
 0\, .\notag
\end{align}
For first inequality above, we note that maximizing $i$ in \eqref{RGeq1} in included in the sum in \eqref{RGeq2}. The second inequality follows from Jensen's inequality. 
\Endproof

\subsection{Proof of MaxWeight Stability}

\begin{theorem}
Given $a(t), t\in\mathbb Z_+$ are independent identically distributed with mean $\bar a$ then
\begin{enumerate}[i)]
	\item If $\bar{a}\notin <
{\mathcal S} >$ then, regardless of the policy used, $(Q_j(t) : j=1,...,q)$ is transient.
\item  If $\bar{a}\in <
{\mathcal S} >^\circ$ then, under the MaxWeight policy, $(Q_j(t) : j=1,...,q)$ is positive recurrent.
\end{enumerate}
\end{theorem}

\Beginproof
For Part i), we note that given $\bar a \notin <\mathcal S>$ then there is a hyperplane separating $\bar a$ from $<\mathcal S>$. We let $\hat n\in  \mathbb Z_+$ be the normal vector of this hyperplane. Since this hyperplane is separating there exists an $\hat \epsilon$ such that 
\begin{align*}
  \hat n \cdot ( d - \bar a ) \geq \hat \epsilon
\end{align*}
for all $d\in\mathcal S$. Thus we see that irrespective of the policy used it must hold that $\mathbb E [\hat n \cdot (Q(t) - Q(0))] \geq \hat \epsilon t $, which clearly diverges as $t\rightarrow \infty$.  This proves Part i).  

We now focus on Part ii). Here the crux of this argument is that for all queue sizes $q$, $q\cdot a$ is by some
margin strictly less than the MaxWeight choice $\max_{\sigma\in <{\mathcal S}>} q\cdot \sigma$. Given
$\epsilon$ above,
we aim to use  $\hat{\epsilon}$, the biggest such that $\bar{a}(t)+\hat{\epsilon}\mathbf{1} \in <
{\mathcal S} >$ for all time. This can be defined as the biggest $\hat{\epsilon}$ such that
$\hat{\epsilon}\leq \epsilon(t):= \max_{\sigma\in<{\mathcal S}>} \min_{j\in{\mathcal J}} (\sigma_j - \bar{a}_j(t))$
for all $t\in{\mathbb Z}_+$ i.e. so the distance of the smallest component to the boundary is maximized.
Given this optimization description we observe the following
\begin{equation}
 \hat{\epsilon}\leq \max_{\sigma\in{\mathcal S}} \min_{j\in{\mathcal J}} (\sigma_j - \bar{a}_j(t)) = \max_{\sigma\in<{\mathcal S}>} \min_{q:
\sum_{j\in{\mathcal J}} q_j=1} q\cdot (\sigma - \bar{a}(t))\leq \min_{q:
\sum_{j\in{\mathcal J}} q_j=1} \max_{\sigma\in<{\mathcal S}>} q\cdot (\sigma - \bar{a}(t)). \footnote{We inequality
in the expression holds with equality by the Minimax Theorem; however, we only
require this easier-to-prove inequality.}
\end{equation}
From this observation, we see one way that we can use $\hat{\epsilon}$ to bound the gap between
$q\cdot \bar{a}(t)$ and $\max_{q\in<{\mathcal S}>} q\cdot \sigma$.
\begin{align*}
& {\mathbb E} ||Q(t+1)||^2 - {\mathbb E} ||Q(t)||^2  \\ 
& =  2 {\mathbb E}\left[ Q(t)\cdot (a(t+1)-\sigma(t+1)) \right] +{\mathbb E} ||
a(t+1)-\sigma(t+1)||^2\\
&=   2{\mathbb E}\left[ Q^\Sigma(t) \frac{Q(t)}{Q^\Sigma(t)}\cdot (\bar{a}(t+1)-\sigma(t+1))\right]  
+ 2{\mathbb E}\left[Q(t)\cdot (a(t+1)-\bar{a}(t+1))\right]
+ c\\
%
%& =   2 \frac{Q(t)}{Q^\Sigma(t)}\cdot (\bar{a}(t+1)-\sigma(t+1)) Q^\Sigma(t)\\
%
& \leq  - 2{\mathbb E} \left[ Q^\Sigma(t) \right]\min_{q: \sum_{j} q_j=1} \max_{\sigma\in <{\mathcal S}>} q\cdot
(\sigma-\bar{a}(t+1)) + c\\
&\leq  - 2 {\mathbb E} \left[ Q^\Sigma(t) \right] \max_{\sigma\in <{\mathcal S}>} \min_{q: \sum_{j}
q_j=1} q\cdot (\sigma-\bar{a}(t+1)) + c\\
&\leq - 2\hat{\epsilon} {\mathbb E} \left[ Q^\Sigma(t) \right] + c.
\end{align*}
Summing the above expression from $t=0,...,T-1$, we gain the expression
\begin{align*}
- {\mathbb E} ||Q(0)||^2 \leq {\mathbb E} ||Q(T)||^2 - {\mathbb E} ||Q(0)||^2 \leq -2 \hat{\epsilon}  {\mathbb E} \left[
\sum_{t=0}^{T-1} Q^\Sigma(t) \right] + cT.
\end{align*}
Rearranging, dividing by $T$, and letting $T\rightarrow\infty$, we obtain the expression:
\begin{equation}
 \limsup_{T\rightarrow\infty}{\mathbb E}\left[ \frac{1}{T}\sum_{t=0}^{T-1} Q^\Sigma(t) \right] \leq 
\frac{c}{2\hat{\epsilon}}.\label{MW:ineq}
\end{equation}
which can only hold if the set of states $\{ Q^\Sigma(t) \leq \frac{c}{2\hat{\epsilon}} \}$ is positive recurrent. 
\Endproof

\subsection{Regret and Service in a Queue}

\begin{lemma} If $t_0$ is the last time before $t$ where the queue for policy $\pi$ is empty, i.e. $W^{\pi}_{t_0} = 0$, then
\begin{align}\label{W:regret}
  W^{\pi}_{t+1} -  W^{\star}_{t+1} \leq  \mathcal R\! g(t) - \mathcal R\! g(t_0-1) \, .
\end{align}

\end{lemma}
\Beginproof
	Note that $Q^\pi_t \geq0 $ for $t\geq t_0$. Thus under the Lindley recursion it holds that
\begin{align}\label{W:pi}
  W_{t+1}^\pi & = W^\pi_t + \tau_t^\pi - a_t
=
...
= \sum_{s=t_0}^t (\tau^\pi_s - a_s) \, .
\end{align}
Similarly repeated substitution for $W^\pi_t$ gives
\begin{align*}
  W^\star_{t+1} = 0 \vee \left\{ W_t^\star + \tau^\star_t - a_t \right\}
=
...
=
\bigvee_{t'=0}^t \left\{ 
\sum_{s=t-t'}^t
	(\tau^\star_s - a_s )
\right\} 
\end{align*}
The maximum above is less than the value attained at $t'=t-t_0$ and so
\begin{align}\label{W:star}
  W^\star_{t+1} \geq \sum^t_{s=t_0}(\tau^\star_s - a_s )
\end{align}
Combining the inequalities \eqref{W:pi} and \eqref{W:star} then gives
\begin{align*}
  W_{t+1}^\pi - W_{t+1}^\star \leq \sum_{s=t_0}^t (\tau^\pi_s - \tau^\star_s) = \mathcal R\! g(t) - \mathcal R\! g(t_0-1)
\end{align*}
as required.
\Endproof

\begin{proposition}[OCO bound]
	 For any function $\nabla l_t(w)$ and $w_t$, as above, then for all $w \in \mathbb R^p$
	 \[
	 \sum_{t=1}^{T} l_t(w_t) - \sum_{t=1}^{T}  l_t(w)
	 \leq 
\frac{|| w_1 -w ||}{\alpha_1}+
	 \sum_{t=1}^{T} \frac{||w_t - w ||^2}{2}
	 	\left(\frac{1}{\alpha_t}-\frac{1}{\alpha_{t+1}}\right)
	 	+
	 	\sum_{t=1}^{T} \frac{\alpha_t}{2} || \nabla l_t(w_t) ||^2 \, .
	 \]
\end{proposition}
\Beginproof
	For any $w\in \mathbb R^p$,
	\begin{align*}
		|| w_{t+1}- w ||^2 
		&
		=
		\left|\left| 
			 w_t - \alpha_t \nabla l_t(w_t)  
			-
			w
		\right|\right|^2
		\\
		&
		\leq
		|| w_t - w ||^2 
		+ 
		\alpha_t^2 || \nabla l_t(w_t) ||^2
		-
		2 \alpha_t \left( w_t - w \right) \cdot \nabla l_t(w_t)\, .
	\end{align*}
Rearranging and summing gives
\begin{align}
&
	\sum_{t=1}^{T} \left( w_t - w \right) \cdot \nabla l_t(w_t)
\notag	
\\
	\leq 
	&
	\sum_{t=1}^{T} 
		\left[		
			\frac{1}{2\alpha_t} || w_t - w ||^2 
			-\frac{1}{2\alpha_t} || w_{t+1}- w ||^2 
		\right]
	+
	\sum_{t=1}^{T}
		\frac{\alpha_t}{2} 
		|| \nabla l_t (w_t) ||^2
\notag	\\
	=
	&
	- 
	\frac{1}{2\alpha_T} || w_{T+1}- w ||^2
	+ 
	\sum_{t=1}^{T}
		||w_t - w ||^2
		\left[
			\frac{1}{2\alpha_t} - \frac{1}{2\alpha_{t-1}}
		\right]
		+
	\sum_{t=1}^{T}
	\frac{\alpha_t}{2} 
	|| \nabla l_t (w_t) ||^2 
\notag	
	\\
	\leq 
	&
	\sum_{t=1}^{T}
	||w_t - w ||^2
	\left[
	\frac{1}{\alpha_t} - \frac{1}{\alpha_{t-1}}
	\right]
	+
	\sum_{t=1}^{T}
	\frac{\alpha_t}{2} 
	|| \nabla l_t (w_t) ||^2 \, .
\label{LGrad}
\end{align}
In the equality above we rearrange the order of summation in the square brackets (wlog. we take $||\theta_1 - \theta||^2/\alpha_0 := 0$). 
By convexity
\begin{align}\label{LConv}
  \sum_{t=1}^T l_t(w_t) - l_t(w) \leq \sum_{t=1}^{T} \left( w_t - w \right) \cdot \nabla l_t(w_t)
\end{align}
The inequalities \eqref{LGrad} and \eqref{LConv} gives the required bound.
\Endproof

\begin{theorem}[Perceptron Mistake Bound]
	If there exists $w^\star \in \mathbb R^p$ such that $y^{(t)}\langle x^{(t)},w^{\star} \rangle  \geq 1$ then the perceptron algorithm is such that 
\begin{align*}
  \sum_{t=1}^\infty \hat l_t(w_t)   \leq {D}^2 ||w^\star||^2\, .
\end{align*}
\end{theorem}
\Beginproof
For $w^\star$ as stated, it holds that $\sum_{t=1}^{T}  l_t(w^\star) = 0$. Further, as discussed above, the loss $\hat l_t$ is less than the hinge loss $l_t$. With this and Proposition \ref{OCO:Lem2}, it holds that
\begin{align}
  \sum_{t=1}^T \hat l_t(w_t)	
&
\leq \sum_{t=1}^{T} l_t(w_t) - \sum_{t=1}^{T}  l_t(w^\star)
\notag
\\
&
=
\frac{|| w^\star ||^2}{2\alpha}+
	 \sum_{t=1}^{T} \frac{||w_t - w^\star ||^2}{2}
	 	\left(\frac{1}{\alpha}-\frac{1}{\alpha}\right)
	 	+
	 	\sum_{t=1}^{T} \frac{\alpha}{2} || \nabla l_t(w_t) ||^2 \, 
\notag
\\
&
=
\frac{|| w^\star ||^2}{2\alpha}
+
\frac{\alpha}{2} \sum_{t=1}^T ||y_t x_t ||^2 \hat l_t(w_t) 
\notag
\\
&
\leq 
\frac{|| w^\star ||^2}{2\alpha} + \frac{ \alpha}{2} D^2  \sum_{t=1}^T \hat l_t(w_t) 
\label{wbound}
\end{align}
Notice that the number of mistakes is determined by the sign of $\langle w,x \rangle $ not the magnitude of $w$. This the $\sum_{t=1}^T \hat l_t(w)$ does not depend on $\alpha$. Optimizing over $\alpha$ in the bound \eqref{wbound} gives choice $\alpha = ||w^\star||/(D\sqrt{\sum_{t=1}^T \hat l_t})$, which applying to the above inequality and rearranging yields:
\begin{align*}
  \sum_{t=1}^T \hat l_t(w_t) 
\leq ||w^\star||^2 D^2 \, .
\end{align*}
We can now let $T\rightarrow \infty$ for the required bound.
\Endproof

\subsubsection*{Acknowledgement.}
The authors are grateful to Yash Kanoria making a number of helpful pointers and corrections to the results in Section 2. We are also thankful to two anonymous referees for their careful reading and their suggestions that have helped the structure and readability of the text. 

\bibliographystyle{plain}  % put ./TutORials.bst if using locally
\bibliography{mybibfile,kx}     

\end{document}